\documentclass[letterpaper]{article} 
\usepackage[preprint]{aaai2027}
\usepackage[hyphens]{url}  
\usepackage{graphicx} 
\usepackage{natbib}  
\usepackage{caption} 
\usepackage{algorithm}
\usepackage{algpseudocode} 
\usepackage{subcaption}

\usepackage{newfloat}
\usepackage{listings}
\DeclareCaptionStyle{ruled}{labelfont=normalfont,labelsep=colon,strut=off} 
\floatstyle{ruled}
\newfloat{listing}{tb}{lst}{}
\floatname{listing}{Listing}

\usepackage{booktabs}

\usepackage{amsmath}
\usepackage{amssymb}
\usepackage{amsthm}
\usepackage{mathtools} 
\usepackage{multirow} 

\usepackage[capitalize]{cleveref}
\crefname{section}{Sec.}{Secs.}
\Crefname{section}{Section}{Sections}
\crefname{table}{Tab.}{Tabs.}
\Crefname{table}{Table}{Tables}
\newtheorem{theorem}{Theorem}
\newtheorem{proposition}{Proposition}

\crefname{theorem}{Thm.}{Thms.}
\Crefname{theorem}{Theorem}{Theorems}
\crefname{proposition}{Prop.}{Props.}
\Crefname{proposition}{Proposition}{Propositions}
\crefname{lemma}{Lemma}{Lemmas}
\Crefname{lemma}{Lemma}{Lemmas}

\usepackage{xspace}
\makeatletter
\DeclareRobustCommand\onedot{\futurelet\@let@token\@onedot}
\def\@onedot{\ifx\@let@token.\else.\null\fi\xspace}
\def\eg{\emph{e.g}\onedot}

\def\ie{\emph{i.e}\onedot}

\makeatother

\title{Hyper-FSAD: Training-Free and Language-Free Few-Shot Anomaly Detection \\ via Sparse Hyper Matching}
\author{
    Guohuan Xie\textsuperscript{\rm 1},
    Xin He\textsuperscript{\rm 2},
    Dingying Fan\textsuperscript{\rm 1},
    Siqi Li\textsuperscript{\rm 3}\corresponding,
    Yun Liu\textsuperscript{\rm 1,4,5}\corresponding
}
\affiliations{
    \textsuperscript{\rm 1}VCIP, College of Computer Science, Nankai University\\
    \textsuperscript{\rm 2}School of Computer Science and Engineering, Tianjin University of Technology\\
    \textsuperscript{\rm 3}School of Software, Tsinghua University\\
    \textsuperscript{\rm 4}Academy for Advanced Interdisciplinary Studies, Nankai University\\
    \textsuperscript{\rm 5}Nankai International Advanced Research Institute, Shenzhen Futian
}

\begin{document}

\maketitle


\begin{abstract}
    Few-shot anomaly detection (FSAD) is particularly valuable when only a few normal images are available in a new target domain, while anomalous cases are rare, diverse, and difficult to enumerate in advance. However, existing methods often still require task-specific fitting or language prompts, and their patch-level retrieval commonly relies on brittle nearest-neighbor or fixed Top-$p$ rules. We propose \textbf{Hyper-FSAD}, a training-free and language-free framework that performs support-only inference with a frozen visual encoder. We formulate FSAD as support-only scoring in frozen feature space and establish the selective stability of sparse retrieval and sufficient conditions for normal--anomaly separation. Guided by this analysis, \textbf{Sparse Hyper Matching} uses \textit{sparsemax} to adaptively select support patches for each query patch, exactly suppressing below-threshold distractors without a manually specified retrieval hyperparameter. \textbf{Dual-Branch Image Scoring} further combines local reconstruction discrepancies with support-conditioned global <CLS> deviation. Across four industrial and two medical benchmarks, Hyper-FSAD achieves the best overall performance across the 1/2/4-shot settings, while requiring only 52.6\,ms per image and 0.89\,GB GPU memory. The code will be released.
\end{abstract}

\section{Introduction}


Few-shot anomaly detection (FSAD) aims to detect anomalies in images and identify potentially anomalous regions with only a few normal samples available for support. 
It is particularly important in real-world applications where representative target-domain data are scarce and anomalous cases are rare, diverse, and difficult to collect or enumerate in advance, such as medical image analysis~\cite{mahapatra2021medical,menze2014multimodal} and industrial defect detection~\cite{qu2023investigating,chen2025center}.

Existing FSAD methods generally follow two paradigms: source-domain training and target-domain training. Source-domain training typically trains a model on a dataset's training set, potentially augmented with synthetic anomalies, and then transfers it to new datasets for testing~\cite{tao2025kernel,ma2025remp,huang2022registration,zhou2024anomalyclip,chen2025center}. In target-domain training, a prompt or adapter is learned on the target dataset using only normal samples before testing~\cite{li2024promptad,fang2023fastrecon,jeong2023winclip}. 
However, both paradigms remain constrained in real-world deployment because they rely on training or adaptation under specific data distributions.

\begin{figure}[!t]
    \centering
    \includegraphics[width=0.90\linewidth]{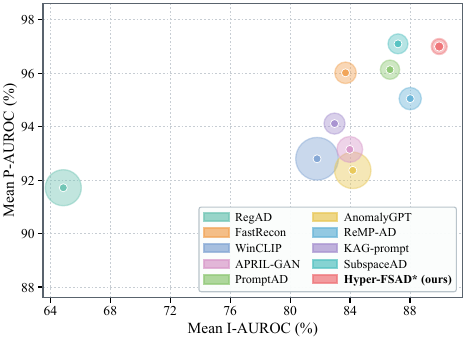}
    \caption{One-shot accuracy--efficiency comparison of Hyper-FSAD with representative FSAD methods across six datasets. Each point reports the mean image- and pixel-level AUROC, while bubble area is log-scaled by per-image runtime (smaller is faster). Hyper-FSAD$^\ast$ denotes our DINOv3 ViT-B/16 instantiation.}
    \label{fig:1shot_accuracy_efficiency}
    \vspace{-3mm}
\end{figure}

In practice, target-domain distributions may evolve continuously and are difficult to model explicitly.
For example, lighting fluctuations or material batch differences can introduce implicit distribution shifts~\cite{guo2024impact,jsss-14-119-2025,wang2025industrial}, causing trained models to degrade during deployment. Moreover, data privacy and cross-domain sharing restrictions can make repeated retraining or domain adaptation difficult to implement. These challenges limit the robustness and deployability of training-dependent FSAD paradigms.

To improve generalization under limited target-domain data, recent FSAD methods have increasingly leveraged the language-derived semantic priors of large-scale vision-language models, particularly CLIP \cite{radford2021learning}, demonstrating impressive performance across various domains \cite{jeong2023winclip,cao2024adaclip,gu2024anomalygpt,qu2024vcp,zhou2024anomalyclip,li2024promptad,tao2025kernel,ma2025remp}. However, these methods often rely on complex prompt engineering or additional training strategies, which still limit their flexibility and deployability. Moreover, their dependence on textual descriptions restricts their applicability to purely visual encoders such as DINOv3 \cite{simeoni2025dinov3}. This raises an important question: Do we really need language supervision for FSAD?

Beyond their reliance on training and language supervision, existing FSAD methods also face limitations in patch-level visual matching. Most methods match each query patch to its nearest normal support patch and convert the similarity into an anomaly score~\cite{jeong2023winclip,li2024promptad,tao2025kernel}. However, under limited normal references, such one-to-one matching is easily affected by background clutter, noise, and local appearance variations, leading to accidental matches or mismatches. This issue is especially evident in medical and industrial images, where anomalous regions may resemble normal tissues or materials, while normal regions can exhibit large intra-class variations. Some methods therefore adopt Top-$p$ matching to aggregate multiple support patches~\cite{ma2025remp}. Yet, its performance depends heavily on the manually specified $p$: a small $p$ provides insufficient normal evidence, whereas a large $p$ may introduce irrelevant patches and dilute anomaly cues. These limitations motivate a data-adaptive matching mechanism that selects relevant normal evidence without a fixed size.

To address these challenges, we take a theory-guided approach. We formulate FSAD as a generic support-only scoring problem in frozen feature space and derive sufficient normal-coverage and anomaly-margin conditions for separation, an exponential few-shot coverage bound, and selective-stability guarantees for retrieval as the memory grows. These results lead to \textbf{Hyper-FSAD}, a training-free and language-free framework that requires no target-task optimization, text input, or prompting, and is compatible with frozen visual backbones, including the CLIP image encoder~\cite{radford2021learning}, DINOv2~\cite{oquab2023dinov2}, and DINOv3~\cite{simeoni2025dinov3}. Here, \emph{Hyper} denotes a one-to-set matching structure: instead of pairing a query patch with a single support patch, each query patch and an adaptively selected set of normal support patches jointly form a higher-order matching unit. We instantiate this view as \textbf{Sparse Hyper Matching}, where \textit{sparsemax}~\cite{martins2016softmax,peters2019sparse} assigns nonzero weights only to a sparse, similarity-dependent subset of relevant support patches---without fixing its size---and reconstructs the query patch from this subset; all below-threshold support patches receive exactly zero weight. Our analysis proves that adding any number of such below-threshold distractors leaves the reconstruction unchanged and that the retrieval weights are non-expansive, while the coverage--margin conditions provide a sufficient guarantee for training-free normal--anomaly separation. Finally, \textbf{Dual-Branch Image Scoring} combines a spatial branch that summarizes patch reconstruction discrepancies with a support-aware <CLS> branch that captures global semantic deviation, producing the final anomaly score using only visual evidence.

Our main contributions are summarized as follows:
\begin{itemize}
    \item We formulate FSAD as a generic support-only scoring problem in frozen feature space. We prove sufficient normal-coverage and anomaly-margin conditions for training-free visual separation, and quantify few-shot global coverage as a function of $K$.
    \item The analysis identifies selective, simplex-valued retrieval as a key requirement. We prove exact below-threshold distractor rejection and non-expansive weights for Euclidean simplex projection, which directly motivates our \textbf{Sparse Hyper Matching}.
    \item Guided by the two additive separation margins, we propose \textbf{Hyper-FSAD}, combining sparse spatial reconstruction with support-conditioned global <CLS> evidence in a training-free and language-free framework.
\end{itemize}

\section{Related work}

FSAD has evolved rapidly, driven by the need to handle open-world defects with minimal supervision. Early approaches like RegAD~\cite{huang2022registration} adopted registration-based paradigms to align test images with a normal reference, while others, such as FastRecon~\cite{fang2023fastrecon}, focused on rapid feature reconstruction to identify deviations. Recently, the success of Vision-Language Models (VLMs) has spurred a wave of methods like WinCLIP~\cite{jeong2023winclip}, AnomalyCLIP~\cite{zhou2024anomalyclip}, PromptAD~\cite{li2024promptad}, and KAG-prompt~\cite{tao2025kernel}, which leverage CLIP's zero-shot capabilities by designing text prompts to distinguish normal from anomalous samples. ReMP-AD~\cite{ma2025remp} further enhances this by fusing multi-modal prompts with retrieval mechanisms. Despite their progress, these methods largely rely on either source-domain meta-training~\cite{huang2022registration} or target-domain training~\cite{li2024promptad}, and thus inherit the practical constraints of training in deployment, such as distribution drift, open-set defect emergence, limited compute budgets, and privacy restrictions that hinder frequent retraining or domain adaptation. Moreover, their dependence on text prompts often introduces a semantic gap when describing complex industrial defects. This issue is especially pronounced when defects are subtle, texture-like, or category-specific, where precise linguistic specification is difficult even for human experts. More recently, SubspaceAD~\cite{lendering2026subspacead} fits a per-category PCA subspace over frozen DINOv2 patch features and scores anomalies via the reconstruction residual. RAD~\cite{zhang2026training} instead uses fixed-memory nearest-neighbor retrieval and analyzes the stability of its distance-to-memory score. Complementarily, Hyper-FSAD studies data-adaptive sparse reconstruction and dual-granularity scoring: it constructs matches directly from visual embeddings, without target-category parameter/subspace fitting or prompt engineering, and our analysis characterizes the resulting adaptive support and few-shot separation.

\section{Theory-Guided Formulation}
\label{sec:theory}
Before specifying Hyper-FSAD, we ask which properties are sufficient for FSAD using only frozen visual features and normal supports. Our analysis first treats retrieval as an unspecified support-only operator; the resulting guarantees then determine the concrete design of Hyper-FSAD.

\subsection{A Generic Support-Only Scoring Family}
Let a frozen encoder produce unit-normalized patch tokens $\mathbf q_{X,t}^l$ and a unit-normalized global token $\mathbf c_X^l$ for image $X$, patch $t$, and layer $l\in L$. The $K$ normal support images induce patch and global memories
$\mathcal P^l=\{\mathbf p_i^l\}_{i=1}^{M}$ and
$\mathcal C^l=\{\mathbf c_k^l\}_{k=1}^{K}$, where $M=KN_p$.
For similarities $z_{X,t,i}^l=\langle\mathbf q_{X,t}^l,\mathbf p_i^l\rangle$, consider any training-free rule
$\rho^l:\mathbb R^M\rightarrow\Delta_M$, where
$\Delta_M=\{\mathbf w\geq\mathbf0:\mathbf1^\top\mathbf w=1\}$.
It yields weights $\mathbf w_{X,t}^l=\rho^l(\mathbf z_{X,t}^l)$ and normal evidence
$\boldsymbol\varepsilon_{X,t}^l=\sum_iw_{X,t,i}^l\mathbf p_i^l$.

Dense weighting accumulates irrelevant evidence, whereas hard nearest-neighbor and Top-$p$ rules predefine the retained neighborhood by a fixed cardinality or proportion and can switch discontinuously at rank boundaries. This motivates the following regularized simplex rule, which remains generic at this stage:
\begin{equation}
\rho_{\mathrm{proj}}(\mathbf z)
=\mathop{\arg\min}_{\mathbf w\in\Delta_M}
\tfrac12\|\mathbf w-\mathbf z\|_2^2.
\label{eq:projection_retrieval}
\end{equation}

\begin{proposition}[Adaptive support and selective stability]
\label{prop:sparse_support}
For a nonempty $R\subseteq\{1,\ldots,M\}$, let
$\tau_R=(\sum_{i\in R}z_i-1)/|R|$. If
\begin{equation}
\max_{j\notin R}z_j\leq\tau_R<\min_{i\in R}z_i,
\label{eq:support_condition}
\end{equation}
then \cref{eq:projection_retrieval} selects exactly $R$, with $w_i=z_i-\tau_R$ on $R$. Appending any number of candidates with similarity at most $\tau_R$ leaves the original weights and reconstruction unchanged. Moreover,
$\|\rho_{\mathrm{proj}}(\mathbf z)-\rho_{\mathrm{proj}}(\mathbf z')\|_2
\leq\|\mathbf z-\mathbf z'\|_2$.
\end{proposition}
\noindent\emph{Proof sketch.}
The simplex KKT conditions give $w_i=[z_i-\tau]_+$ and $\sum_i[z_i-\tau]_+=1$. Under \cref{eq:support_condition}, $\tau_R$ activates exactly $R$; it remains a valid threshold after any below-threshold additions. Non-expansiveness follows because Euclidean projection onto a closed convex set is non-expansive. The full proof is in Appendix~\ref{sec:theory_proofs}.

\subsection{Sufficient Conditions for Frozen-Feature Separation}
For a nonzero reconstruction, define a generic local score, global score, and their convex fusion:
\begin{align}
S_{\mathrm{map}}(X)
&=\max_t\frac1{|L|}\sum_{l\in L}
\left(1-\mathrm{cos}(\mathbf q_{X,t}^l,\boldsymbol\varepsilon_{X,t}^l)\right),
\label{eq:generic_map_score}\\
S_{\mathrm{cls}}(X)
&=\frac1{|L|}\sum_{l\in L}
\left(1-\max_k\langle\mathbf c_X^l,\mathbf c_k^l\rangle\right),
\label{eq:cls_score}\\
S_{\mathrm{image}}(X)
&=\lambda S_{\mathrm{map}}(X)+(1-\lambda)S_{\mathrm{cls}}(X).
\label{eq:final_score}
\end{align}

\begin{theorem}[Sufficient conditions for training-free visual separation]
\label{thm:separation}
Consider a normal image $N$ and an anomalous image $A$. At every $l\in L$, suppose: (i) $w_{N,t,i}^l>0$ implies
$\|\mathbf q_{N,t}^l-\mathbf p_i^l\|_2\leq\epsilon_l<\sqrt2$;
(ii) one patch $t^\star$ of $A$ has a nonzero reconstruction and cosine similarity at most $1-\gamma_l$ to every nonzero $\mathbf v\in\operatorname{conv}(\mathcal P^l)$; and (iii)
$\min_k\|\mathbf c_N^l-\mathbf c_k^l\|_2\leq r_l$ and
$\min_k\|\mathbf c_A^l-\mathbf c_k^l\|_2\geq\delta_l$.
Then
\begin{equation}
S_{\mathrm{image}}(A)-S_{\mathrm{image}}(N)
\geq \lambda\Delta_{\mathrm{map}}+(1-\lambda)\Delta_{\mathrm{cls}},
\label{eq:fused_margin}
\end{equation}
where
$\Delta_{\mathrm{map}}=|L|^{-1}\sum_l(\gamma_l-\epsilon_l^2/2)$ and
$\Delta_{\mathrm{cls}}=(2|L|)^{-1}\sum_l(\delta_l^2-r_l^2)$.
Consequently, a positive right-hand side guarantees separation without target-task fitting or language input.
\end{theorem}
\noindent\emph{Proof sketch.}
Simplex weights make each reconstruction a convex combination. For a normal patch, unit normalization and (i) give cosine similarity at least $1-\epsilon_l^2/2$; (ii) gives anomaly score at least $\gamma_l$. For global tokens,
$1-\langle\mathbf c,\mathbf c_k\rangle=\|\mathbf c-\mathbf c_k\|_2^2/2$.
Averaging over layers, max pooling, and convex fusion yield \cref{eq:fused_margin}; Appendix~\ref{sec:theory_proofs} gives the full proof.

The global coverage condition also exposes the role of shots. For a fixed normal token $\mathbf c_N^l$, if each of the $K$ i.i.d. supports falls within radius $r_l$ with probability at least $\alpha_l$, then
\begin{equation}
\Pr\!\left[\min_k\|\mathbf c_N^l-\mathbf c_k^l\|_2>r_l\right]
\leq(1-\alpha_l)^K\leq e^{-K\alpha_l}.
\label{eq:fewshot_coverage}
\end{equation}
These are sufficient rather than distribution-free conditions; the analysis does not assume that every frozen encoder satisfies the required margins.

\paragraph{From theory to design.}
The analysis yields three requirements. First, local normal evidence should be a convex reconstruction from compatible support patches. Second, increasing the patch memory calls for exact rejection of irrelevant entries; Proposition~\ref{prop:sparse_support} shows that the projection in \cref{eq:projection_retrieval}, known as \emph{sparsemax}~\cite{martins2016softmax}, is an adaptive and stable choice. Third, the additive margins in Theorem~\ref{thm:separation} motivate combining local spatial evidence with support-conditioned global evidence. These conclusions lead directly to Hyper-FSAD.

\section{Method}
\label{sec:method}

\begin{figure*}[t]
  \centering
  \includegraphics[width=\textwidth]{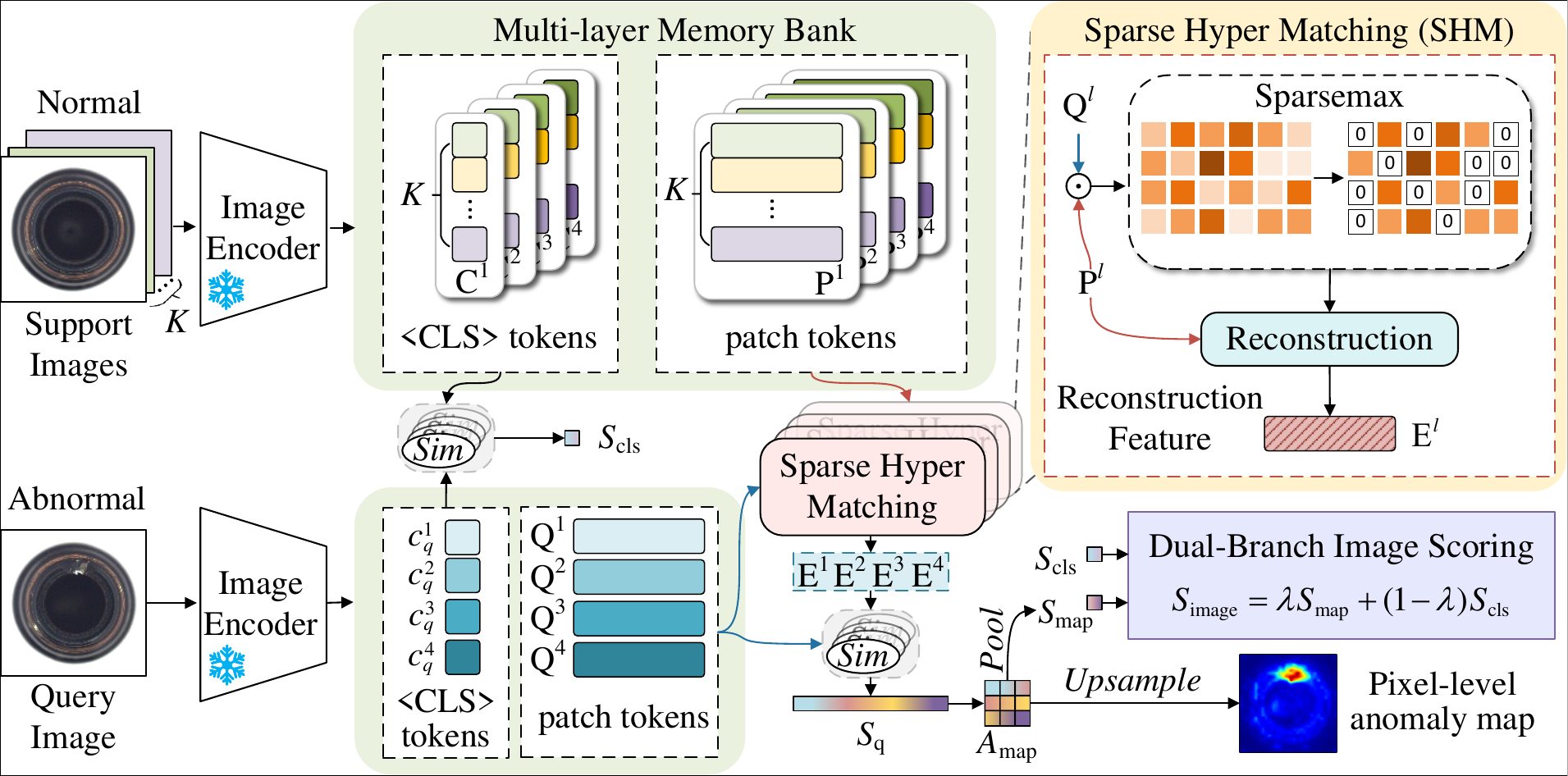}
    \caption{Overview of \textbf{Hyper-FSAD}. Given $K$ normal support images and a query image, a frozen visual encoder extracts multi-layer patch tokens and <CLS> tokens to construct a \emph{Multi-layer Memory Bank} of normal features. For each query patch at layer $l$, we perform \emph{Sparse Hyper Matching}: similarities to all support patches in $\mathbf{P}^l$ are converted into sparse fusion weights via \emph{sparsemax}, selecting a few informative neighbors to form a sparse reconstruction. The pixel-level anomaly map is obtained by the cosine distance between each query patch and its sparse reconstruction, aggregated across layers. \emph{Dual-Branch Image Scoring} further combines a pooled map-based score $S_{\mathrm{map}}$ from the \emph{Spatial Evidence Branch} and a <CLS>-based score $S_{\mathrm{cls}}$ from the \emph{Global Semantic Branch} into the final image-level anomaly score $S_{\mathrm{image}}$.}
  \label{fig:framework}
\end{figure*}

\subsection{Framework Overview}
\label{sec:overview}
Guided by Proposition~\ref{prop:sparse_support} and Theorem~\ref{thm:separation}, we instantiate the generic support-only family as \textbf{Hyper-FSAD}, illustrated in \cref{fig:framework}. Given $K$ normal support images and a query image, a frozen visual encoder extracts multi-layer patch and <CLS> tokens. We store both in a normal memory, instantiate $\rho_{\mathrm{proj}}$ as \textbf{Sparse Hyper Matching}, and implement the two margin terms in Theorem~\ref{thm:separation} through \textbf{Dual-Branch Image Scoring}. No parameters are fitted.

Every token entering a similarity computation is $\ell_2$-normalized, so inner products are cosine similarities. Here, \emph{language-free} means that target-task adaptation and inference require no text input, text encoder, prompt, or language supervision; the DINOv2/v3 instantiations additionally use visually pre-trained backbones.

\subsubsection{Multi-layer Memory Bank.}
We construct a multi-layer memory bank to store multi-layer features of normal samples, leveraging the strong representations of modern frozen visual encoders. Instead of relying on a single layer, which may miss multi-scale cues and thereby limit robustness, we extract features from four selected layers of the backbone, $L=\{l_1, l_2, l_3, l_4\}$, whose patch-token grid size and dimension are identical across layers. For each support image in the few-shot support set $\mathcal{S} = \{I_s^1, I_s^2, \dots, I_s^K\}$, where $K$ denotes the number of normal support images, we collect two distinct types of tokens: patch tokens, which preserve fine-grained local spatial information essential for pixel-level anomaly detection, and the class token <CLS>, which captures global semantic information suitable for image-level anomaly detection. These tokens are then stored in layer-specific memory banks to retain complementary cues across depths. Specifically, for each selected layer $l$, all patch tokens from the support images are concatenated to form $\mathbf{P}^l \in \mathbb{R}^{(K \cdot N_p)\times D}$, where $N_p$ denotes the total number of patch tokens of the feature map at layer $l$ and $D$ is the token dimension. Meanwhile, we store the <CLS> tokens as $\mathbf{C}^l \in \mathbb{R}^{K\times D}$ to represent global normality. This structure allows our model to perform precise matching at both spatial and global granularities without the need for additional training or labor-intensive hand-crafted text prompts, thus improving practicality in deployment.

\subsection{Sparse Hyper Matching}
\label{sec:hypermatch}
Proposition~\ref{prop:sparse_support} motivates instantiating the generic rule $\rho^l$ with the projection $\rho_{\mathrm{proj}}$, \ie, \emph{sparsemax}. For query patches $\mathbf Q^l\in\mathbb R^{N_p\times D}$, each row $\mathbf q^l$ is compared with the layer memory $\mathbf P^l$ to obtain
$\mathbf z^l=\mathbf q^l(\mathbf P^l)^\top\in\mathbb R^{1\times M}$. For completeness, the retrieval alternatives evaluated in our ablation are
\begin{equation}
\hat{\mathbf{w}}^{l}=
\begin{cases}
\mathrm{softmax}(\mathbf{z}^l) & \text{Dense Lookup}, \\
\mathrm{softmax}\!\big((\mathbf{z}^l)_{\text{top-}p}\big) & \text{Top-$p$ Lookup}, \\
\mathrm{sparsemax}(\mathbf{z}^l) & \text{Sparse Lookup}.
\end{cases}
\end{equation}

Dense Lookup weights every patch, whereas Top-$p$ retains only a fixed percentage of patches. Sparse Hyper Matching instead uses the closed-form threshold
$\hat w_u^l=[z_u^l-\tau]_+$ implied by \cref{eq:projection_retrieval}, so its active set is data-dependent and below-threshold patches receive exact zero weight. The computation is detailed in Appendix~\ref{sec:sparsemax_algorithm}.

Finally, $\boldsymbol\varepsilon^l=\hat{\mathbf w}^l\mathbf P^l$ is the sparse normal reconstruction. Applying this rule to every query patch gives $\mathbf E^l\in\mathbb R^{N_p\times D}$.

\subsection{Dual-Branch Image Scoring}
\label{sec:scoring}
The two additive margins in Theorem~\ref{thm:separation} motivate complementary local and global evidence. We implement them as a \emph{spatial evidence branch} based on sparse patch reconstruction and a \emph{global semantic branch} based on nearest-support <CLS> matching.

\textbf{Spatial Evidence Branch.}
Given a query image $I_q$, we extract its patch tokens $\mathbf{Q}^l \in \mathbb{R}^{N_p \times D}$
and the corresponding sparse reconstructions $\mathbf{E}^l \in \mathbb{R}^{N_p \times D}$ at each selected layer $l\in L$
using Sparse Hyper Matching. We compute the patch-wise anomaly score vector by averaging the cosine distance across layers:
\begin{equation}
\mathbf{S}_q = \frac{1}{|L|}\sum_{l\in L}\left(\mathbf{1} - \mathrm{cos}\!\left(\mathbf{Q}^l,\mathbf{E}^l\right)\right)\in\mathbb{R}^{N_p},
\label{eq:patch_score_vec}
\end{equation}
where $\mathrm{cos}(\mathbf{Q}^l,\mathbf{E}^l)$ denotes row-wise cosine similarity between corresponding rows of $\mathbf{Q}^l$ and $\mathbf{E}^l$.
We reshape $\mathbf{S}_q$ into a patch-grid anomaly map $A_{map}\in\mathbb{R}^{H_p\times W_p}$ with $H_pW_p=N_p$.

We further reduce $A_{map}$ to a scalar score through a pooling operator $f_{\mathrm{pool}}(\cdot)$, which can be instantiated as max, Top-$n$, or Top-$p$ pooling; unless otherwise stated, we use max pooling in all experiments (see Appendix~\ref{sec:pooling_appendix}), yielding map-based anomaly score $S_{\mathrm{map}} = f_{\mathrm{pool}}(A_{map})$.

\textbf{Global Semantic Branch.}
For each layer, we match the query <CLS> token $\mathbf c_q^l$ to the $K$ normal tokens in $\mathbf C^l$ and average their nearest-neighbor cosine distances, exactly implementing \cref{eq:cls_score}. Finally, \cref{eq:final_score} fuses $S_{\mathrm{map}}$ and $S_{\mathrm{cls}}$; we use $\lambda=0.5$ throughout.

\section{Experiments}
\label{sec:experiments}

\subsection{Experimental Settings}

\subsubsection{Datasets.}
Experiments are conducted on six real-world datasets spanning industrial inspection and medical imaging. We adopt four industrial anomaly detection benchmarks: MVTecAD~\cite{bergmann2019mvtec}, VisA~\cite{zou2022spot}, MPDD~\cite{jezek2021deep}, and BTAD~\cite{mishra2021vt}, together with two medical datasets, RESC~\cite{hu2019automated} for retinal lesion detection and BraTS~\cite{menze2014multimodal} for brain tumor analysis. Hyper-FSAD follows a strict \emph{no-fitting} protocol: for each test category, we use $K$ normal images as the support set and directly infer anomaly maps and image-level anomaly scores via memory-based matching, without auxiliary data, prompt/adaptor learning, or target-category parameter/subspace estimation. For fair comparison, we reproduce prior methods under their \emph{original training settings} (source-domain or target-domain training, including prompt/adaptor learning when applicable), following the configurations reported in their papers.

\subsubsection{Evaluation metrics.}
Industrial visual anomaly detection is evaluated using the area under the receiver operating characteristic (AUROC), the F1-score at the optimal threshold (F1), average precision (AP), and per-region overlap (PRO). For image-level anomaly classification, we report I-AUROC, I-F1, and I-AP, while pixel-level anomaly segmentation is assessed using P-AUROC, PRO, and AP. Due to space constraints, we present only I-AUROC and P-AUROC in the main paper, and report results for the remaining metrics in Appendix~\ref{sec:metrics_detail}. We report performance with $K$ set to 1, 2, and 4 normal support images.

\subsubsection{Implementation details.}
Since Hyper-FSAD is agnostic to the choice of frozen visual encoder, we instantiate it with three representative pre-trained backbones to demonstrate its generality: the visual encoder of CLIP ViT-L/14@336px~\cite{radford2021learning}, DINOv2 ViT-B/14~\cite{oquab2023dinov2}, and DINOv3 ViT-B/16~\cite{simeoni2025dinov3}. Unless otherwise specified, we adopt dinov3\_vitb16 as the default backbone. For the default backbone, all input images are resized to $448 \times 448$, and, following our multi-layer design, we extract features from transformer layers $L=\{3,6,9,12\}$; the other two backbones follow the same four-layer multi-scale design at their standard input resolutions. For every backbone, we then construct layer-wise memory banks by storing support-set patch tokens and <CLS> tokens for subsequent matching and scoring. Our method is strictly training-free: no auxiliary data, prompt/adaptor learning, or parameter optimization is performed on either source or target datasets; inference is carried out directly using the proposed Sparse Hyper Matching and Dual-Branch Image Scoring. All experiments are conducted on a single NVIDIA RTX 3090 GPU with 24\,GB memory, using a test batch size of 8.

\subsection{Comparison with State-of-the-art Methods}

\begin{table*}[t]
    \centering
    \caption{Comparison with existing state-of-the-art methods across six datasets under 1-shot/2-shot/4-shot settings. The best results are in \textbf{bold} and the second-best are \underline{underlined}. \textbf{I} denotes I-AUROC~(\%) and \textbf{P} denotes P-AUROC~(\%). $\dag$, $\ddag$, and $\ast$ denote our Hyper-FSAD using CLIP ViT-L/14@336px, DINOv2 ViT-B/14, and DINOv3 ViT-B/16 as the backbone, respectively. The $\ddag$ on SubspaceAD indicates that it is likewise based on the DINOv2 ViT-B/14 backbone. Results for Hyper-FSAD (ours) are reported as mean $\pm$ standard deviation over 5 runs.}
    \setlength{\tabcolsep}{8pt}%
    \renewcommand{\arraystretch}{0.85}%
    \resizebox{\textwidth}{!}{%
    \begin{tabular}{ll|l|cc|cc|cc|cc|cc|cc}
    \toprule
    & \multirow{2}{*}{Methods} & \multirow{2}{*}{Venue/Year} &
    \multicolumn{2}{c|}{MVTecAD} & \multicolumn{2}{c|}{VisA} & \multicolumn{2}{c|}{MPDD} &
    \multicolumn{2}{c|}{BTAD} & \multicolumn{2}{c|}{RESC} & \multicolumn{2}{c}{BraTS} \\
    \cmidrule(lr){4-5}\cmidrule(lr){6-7}\cmidrule(lr){8-9}\cmidrule(lr){10-11}\cmidrule(lr){12-13}\cmidrule(lr){14-15}
     & & & I & P & I & P & I & P & I & P & I & P & I & P \\
    \midrule
    \multirow{12}{*}{\rotatebox[origin=c]{90}{1-shot}} & RegAD & ECCV'22 & 73.3 & 92.3 & 69.3 & 93.3 & 47.9 & 93.2 & 84.4 & 95.6 & 55.9 & 84.6 & 58.4 & 91.3 \\
     & FastRecon & ICCV'23 & 92.0 & 95.1 & 81.0 & 96.6 & 76.5 & 96.2 & 93.7 & 96.8 & 82.8 & 96.0 & 76.2 & 95.4 \\
     & WinCLIP & CVPR'23 & 92.6 & 91.6 & 84.8 & 95.3 & 79.9 & 95.6 & 89.5 & 88.9 & 57.4 & 92.3 & 86.6 & 93.1 \\
     & APRIL-GAN & CVPR'23 & 91.1 & 91.2 & 87.1 & 95.9 & 75.1 & 94.9 & 86.5 & 93.0 & 77.3 & 93.0 & 86.8 & 90.9 \\
     & PromptAD & CVPR'24 & 93.0 & 95.2 & 85.2 & 97.2 & 79.3 & 96.0 & 93.4 & 96.1 & 87.4 & \textbf{96.4} & 81.7 & 95.9 \\
     & AnomalyGPT & AAAI'24 & 92.8 & 95.3 & 86.4 & 87.4 & 72.4 & 95.6 & 93.6 & 95.7 & 86.8 & 86.0 & 73.1 & 94.2 \\
     & ReMP-AD & ICCV'25 & 94.7 & 95.6 & 91.3 & 97.6 & 78.6 & 96.0 & 94.8 & 96.5 & 81.6 & 91.4 & 87.1 & 93.2 \\
     & KAG-prompt & AAAI'25 & 95.8 & 96.2 & 91.6 & 97.0 & 68.3 & 96.0 & 89.4 & 96.0 & 69.2 & 83.0 & 83.5 & \underline{96.5} \\
     & SubspaceAD$^\ddag$ & CVPR'26 & 93.3 & 97.2 & 93.2 & \textbf{98.0} & 77.8 & \textbf{97.7} & \underline{94.9} & 97.7 & 73.3 & 94.3 & \textbf{90.7} & \textbf{97.7} \\
     & Hyper-FSAD$^\dag$ & \multicolumn{1}{c|}{-} & 93.1\textsubscript{$\pm$0.4} & 93.4\textsubscript{$\pm$0.1} & 92.4\textsubscript{$\pm$0.7} & 95.0\textsubscript{$\pm$0.1} & \underline{82.8\textsubscript{$\pm$0.7}} & 96.5\textsubscript{$\pm$0.1} & 87.6\textsubscript{$\pm$1.5} & 96.9\textsubscript{$\pm$0.1} & \underline{89.5\textsubscript{$\pm$1.6}} & 90.8\textsubscript{$\pm$0.4} & \underline{89.9\textsubscript{$\pm$1.2}} & 89.3\textsubscript{$\pm$0.3} \\
     & Hyper-FSAD$^\ddag$ & \multicolumn{1}{c|}{-} & \underline{95.9\textsubscript{$\pm$0.1}} & \underline{97.4\textsubscript{$\pm$0.0}} & \underline{93.5\textsubscript{$\pm$0.1}} & 97.3\textsubscript{$\pm$0.1} & \textbf{83.0\textsubscript{$\pm$0.9}} & \underline{97.2\textsubscript{$\pm$0.2}} & \textbf{95.4\textsubscript{$\pm$1.1}} & \underline{97.8\textsubscript{$\pm$0.0}} & 84.7\textsubscript{$\pm$0.1} & 94.4\textsubscript{$\pm$0.2} & 82.0\textsubscript{$\pm$1.7} & 95.5\textsubscript{$\pm$0.2} \\
     & Hyper-FSAD$^\ast$ & \multicolumn{1}{c|}{-} & \textbf{96.3\textsubscript{$\pm$0.2}} & \textbf{97.7\textsubscript{$\pm$0.0}} & \textbf{94.0\textsubscript{$\pm$0.1}} & \underline{97.7\textsubscript{$\pm$0.1}} & 78.8\textsubscript{$\pm$1.5} & 96.4\textsubscript{$\pm$0.1} & \textbf{95.4\textsubscript{$\pm$1.2}} & \textbf{97.9\textsubscript{$\pm$0.0}} & \textbf{93.0\textsubscript{$\pm$1.4}} & \underline{96.3\textsubscript{$\pm$0.0}} & 82.2\textsubscript{$\pm$1.7} & 96.0\textsubscript{$\pm$0.2} \\
    \midrule
    \multirow{12}{*}{\rotatebox[origin=c]{90}{2-shot}} & RegAD & ECCV'22 & 76.6 & 94.5 & 70.4 & 94.3 & 52.5 & 94.0 & 88.9 & 96.9 & 59.4 & 85.9 & 57.4 & 92.7 \\
     & FastRecon & ICCV'23 & 94.2 & 95.5 & 81.1 & 96.6 & 81.8 & 96.5 & 93.8 & 97.2 & 87.6 & 96.2 & 75.8 & 95.2 \\
     & WinCLIP & CVPR'23 & 93.8 & 91.9 & 83.5 & 95.7 & 81.5 & 96.5 & 90.7 & 89.6 & 60.3 & 92.7 & 87.0 & 93.0 \\
     & APRIL-GAN & CVPR'23 & 90.1 & 91.6 & 86.6 & 96.1 & 75.1 & 95.1 & 86.1 & 93.2 & 78.3 & 93.4 & 87.5 & 90.9 \\
     & PromptAD & CVPR'24 & 95.4 & 95.6 & 85.1 & 97.7 & \underline{83.3} & 96.8 & 92.7 & 96.4 & \underline{89.2} & \textbf{96.7} & 83.0 & 95.8 \\
     & AnomalyGPT & AAAI'24 & 94.4 & 95.9 & 87.2 & 87.7 & 79.7 & 96.3 & 93.4 & 96.0 & 87.8 & 86.3 & 74.9 & 94.1 \\
     & ReMP-AD & ICCV'25 & 95.9 & 96.2 & 92.9 & 97.8 & 83.1 & 96.5 & 95.4 & 96.9 & 85.2 & 93.2 & 88.1 & 93.8 \\
     & KAG-prompt & AAAI'25 & \underline{96.6} & 96.5 & 92.7 & 97.4 & 69.4 & 96.1 & 89.7 & 96.0 & 77.4 & 82.9 & 85.1 & 96.6 \\
     & SubspaceAD$^\ddag$ & CVPR'26 & 94.7 & 97.5 & 93.7 & \textbf{98.3} & 79.4 & \textbf{97.7} & \textbf{96.9} & \underline{98.0} & 74.9 & 95.0 & \underline{89.1} & \textbf{98.0} \\
     & Hyper-FSAD$^\dag$ & \multicolumn{1}{c|}{-} & 93.6\textsubscript{$\pm$0.3} & 94.2\textsubscript{$\pm$0.1} & 94.5\textsubscript{$\pm$0.4} & 95.7\textsubscript{$\pm$0.1} & \textbf{83.9\textsubscript{$\pm$0.6}} & 96.6\textsubscript{$\pm$0.1} & 90.5\textsubscript{$\pm$0.7} & 97.1\textsubscript{$\pm$0.0} & 88.1\textsubscript{$\pm$0.7} & 90.9\textsubscript{$\pm$0.2} & \textbf{90.3\textsubscript{$\pm$1.4}} & 91.2\textsubscript{$\pm$0.6} \\
     & Hyper-FSAD$^\ddag$ & \multicolumn{1}{c|}{-} & 95.8\textsubscript{$\pm$0.2} & \underline{97.6\textsubscript{$\pm$0.1}} & \textbf{94.9\textsubscript{$\pm$0.1}} & 97.8\textsubscript{$\pm$0.1} & 82.2\textsubscript{$\pm$0.6} & \underline{97.1\textsubscript{$\pm$0.1}} & \underline{96.4\textsubscript{$\pm$0.5}} & \textbf{98.1\textsubscript{$\pm$0.1}} & 84.7\textsubscript{$\pm$0.1} & 95.0\textsubscript{$\pm$0.2} & 80.5\textsubscript{$\pm$1.7} & 96.5\textsubscript{$\pm$0.4} \\
     & Hyper-FSAD$^\ast$ & \multicolumn{1}{c|}{-} & \textbf{96.8\textsubscript{$\pm$0.0}} & \textbf{97.8\textsubscript{$\pm$0.1}} & \underline{94.8\textsubscript{$\pm$0.3}} & \underline{98.0\textsubscript{$\pm$0.1}} & 80.3\textsubscript{$\pm$0.5} & 96.4\textsubscript{$\pm$0.1} & 95.8\textsubscript{$\pm$0.7} & \underline{98.0\textsubscript{$\pm$0.0}} & \textbf{91.5\textsubscript{$\pm$1.7}} & \underline{96.5\textsubscript{$\pm$0.2}} & 83.2\textsubscript{$\pm$1.5} & \underline{96.9\textsubscript{$\pm$0.3}} \\
    \midrule
    \multirow{12}{*}{\rotatebox[origin=c]{90}{4-shot}} & RegAD & ECCV'22 & 83.4 & 95.7 & 72.0 & 94.7 & 61.1 & 94.9 & 91.3 & 97.3 & 64.2 & 87.9 & 63.3 & 93.8 \\
     & FastRecon & ICCV'23 & 96.2 & 96.3 & 84.4 & 97.0 & 81.9 & 96.9 & 94.4 & 97.4 & 87.5 & 95.8 & 78.6 & 96.1 \\
     & WinCLIP & CVPR'23 & 95.5 & 92.4 & 85.7 & 96.0 & 84.0 & 97.0 & 91.7 & 90.3 & 63.8 & 93.1 & 87.0 & 93.1 \\
     & APRIL-GAN & CVPR'23 & 91.0 & 92.2 & 87.2 & 96.2 & 76.5 & 95.3 & 86.1 & 93.3 & 78.3 & 93.7 & 88.0 & 91.3 \\
     & PromptAD & CVPR'24 & 95.9 & 96.0 & 87.5 & 97.9 & 84.0 & 97.3 & 92.6 & 96.6 & \textbf{90.2} & \textbf{96.8} & 86.4 & 96.6 \\
     & AnomalyGPT & AAAI'24 & \underline{97.0} & 96.4 & 91.4 & 96.5 & \textbf{85.9} & 96.7 & 93.5 & 96.2 & 88.5 & 86.7 & 79.4 & 95.4 \\
     & ReMP-AD & ICCV'25 & 96.8 & 96.6 & 94.1 & 97.8 & 80.0 & 97.0 & 96.0 & 97.0 & 86.6 & 93.6 & \underline{88.2} & 94.7 \\
     & KAG-prompt & AAAI'25 & \textbf{97.1} & 96.7 & 93.3 & 97.7 & 71.6 & 96.6 & 88.4 & 96.0 & 82.0 & 84.0 & 87.5 & \underline{96.7} \\
     & SubspaceAD$^\ddag$ & CVPR'26 & 95.5 & 97.7 & \underline{94.9} & \textbf{98.4} & 79.8 & \textbf{97.8} & \textbf{97.2} & \underline{98.1} & 82.4 & 95.7 & 88.1 & \textbf{97.4} \\
     & Hyper-FSAD$^\dag$ & \multicolumn{1}{c|}{-} & 94.0\textsubscript{$\pm$0.2} & 95.0\textsubscript{$\pm$0.1} & \underline{94.9\textsubscript{$\pm$0.2}} & 96.0\textsubscript{$\pm$0.0} & \underline{84.2\textsubscript{$\pm$0.8}} & 96.9\textsubscript{$\pm$0.1} & 91.0\textsubscript{$\pm$0.4} & 97.5\textsubscript{$\pm$0.1} & 88.7\textsubscript{$\pm$0.7} & 91.2\textsubscript{$\pm$0.2} & \textbf{90.1\textsubscript{$\pm$0.8}} & 91.6\textsubscript{$\pm$0.6} \\
     & Hyper-FSAD$^\ddag$ & \multicolumn{1}{c|}{-} & 96.1\textsubscript{$\pm$0.2} & \underline{97.8\textsubscript{$\pm$0.0}} & \textbf{95.5\textsubscript{$\pm$0.1}} & 98.1\textsubscript{$\pm$0.0} & 81.7\textsubscript{$\pm$0.5} & \underline{97.4\textsubscript{$\pm$0.1}} & \underline{96.8\textsubscript{$\pm$0.5}} & \textbf{98.2\textsubscript{$\pm$0.1}} & 84.6\textsubscript{$\pm$0.8} & 95.0\textsubscript{$\pm$0.3} & 78.5\textsubscript{$\pm$0.9} & 96.0\textsubscript{$\pm$0.4} \\
     & Hyper-FSAD$^\ast$ & \multicolumn{1}{c|}{-} & 96.9\textsubscript{$\pm$0.1} & \textbf{98.0\textsubscript{$\pm$0.0}} & 94.8\textsubscript{$\pm$0.2} & \underline{98.3\textsubscript{$\pm$0.0}} & 80.3\textsubscript{$\pm$0.7} & 96.5\textsubscript{$\pm$0.1} & 96.2\textsubscript{$\pm$0.6} & \textbf{98.2\textsubscript{$\pm$0.0}} & \underline{89.1\textsubscript{$\pm$1.2}} & \underline{95.9\textsubscript{$\pm$0.2}} & 80.5\textsubscript{$\pm$1.1} & 96.5\textsubscript{$\pm$0.6} \\
    \bottomrule
    \end{tabular}}
    \label{tab:auroc_results}
\end{table*}


\subsubsection{Competing Methods.}
We compare our proposed \textbf{Hyper-FSAD} with nine recent baselines, including RegAD~\cite{huang2022registration}, FastRecon~\cite{fang2023fastrecon}, WinCLIP~\cite{jeong2023winclip}, APRIL-GAN~\cite{chen2023april}, PromptAD~\cite{li2024promptad}, AnomalyGPT~\cite{gu2024anomalygpt}, ReMP-AD~\cite{ma2025remp}, KAG-prompt~\cite{tao2025kernel}, and SubspaceAD~\cite{lendering2026subspacead}. For a fair comparison, all methods are evaluated under identical 1-, 2-, and 4-shot settings on six datasets (MVTecAD~\cite{bergmann2019mvtec}, VisA~\cite{zou2022spot}, MPDD~\cite{jezek2021deep}, BTAD~\cite{mishra2021vt}, RESC~\cite{hu2019automated}, and BraTS~\cite{menze2014multimodal}), using the same sampled support sets, with each experiment repeated five times.

\subsubsection{Quantitative Comparison.}
As summarized in~\cref{tab:auroc_results}, we report image-level AUROC (I-AUROC) and pixel-level AUROC (P-AUROC) under 1/2/4-shot settings for Hyper-FSAD instantiated with three frozen backbones ($\dag$~CLIP ViT-L/14@336px, $\ddag$~DINOv2 ViT-B/14, and $\ast$~DINOv3 ViT-B/16).
Overall, Hyper-FSAD is highly competitive across all three backbones and consistently delivers strong \textit{average} results.
With the default DINOv3 ViT-B/16, Hyper-FSAD reaches average pixel-level AUROC values of 97.00\%/97.27\%/97.23\% for 1/2/4-shot, surpassing \textit{PromptAD} by \textbf{+0.87}, \textbf{+0.77}, and \textbf{+0.36} points, respectively.
For image-level detection, Hyper-FSAD also achieves the best average I-AUROC at every shot, reaching 89.95\%/90.40\% with DINOv3 under 1/2-shot and 90.48\% with CLIP under 4-shot, compared with 88.02\%/90.10\%/90.28\% for the strongest baseline (\textit{ReMP-AD}).
The aggregate trade-off in \cref{fig:1shot_accuracy_efficiency} further shows that the default Hyper-FSAD combines 89.95\%/97.00\% mean 1-shot I/P-AUROC with a 1-shot runtime of 52.6\,ms, lower than all compared baselines in \cref{tab:efficiency_comparison}.
Crucially, all three instantiations remain on par with or superior to prior training-based, language-guided methods, showing that the improvements stem from our Sparse Hyper Matching and Dual-Branch Image Scoring rather than from a particular backbone.
It is worth noting that some competing methods further require additional training or per-task adaptation when applied to new categories or datasets.
In contrast, Hyper-FSAD performs inference in a \emph{training-free} manner, enabling a simpler deployment pipeline while maintaining superior accuracy. Additional support images generally improve normal coverage, although the gains need not be monotonic for every dataset or metric.

\subsubsection{Qualitative Comparison.}
\cref{fig:qualitative} visualizes anomaly localization results under the 1-shot setting on both industrial and medical datasets.
Overall, \textbf{Hyper-FSAD} consistently produces noticeably more accurate anomaly maps with clearer boundaries and fewer false positives in normal regions, while preserving fine-grained abnormal structures.
Across these representative examples, Hyper-FSAD concentrates responses tightly within annotated anomalies while effectively suppressing background activations.
Additional Hyper-FSAD$^\ast$--GT comparisons in Appendix~\ref{sec:more_qualitative} further demonstrate reliable pixel-level localization across diverse anomaly appearances under limited normal references.

\begin{figure}[!b]
    \centering
    \vspace{-2pt}
    \includegraphics[width=\linewidth]{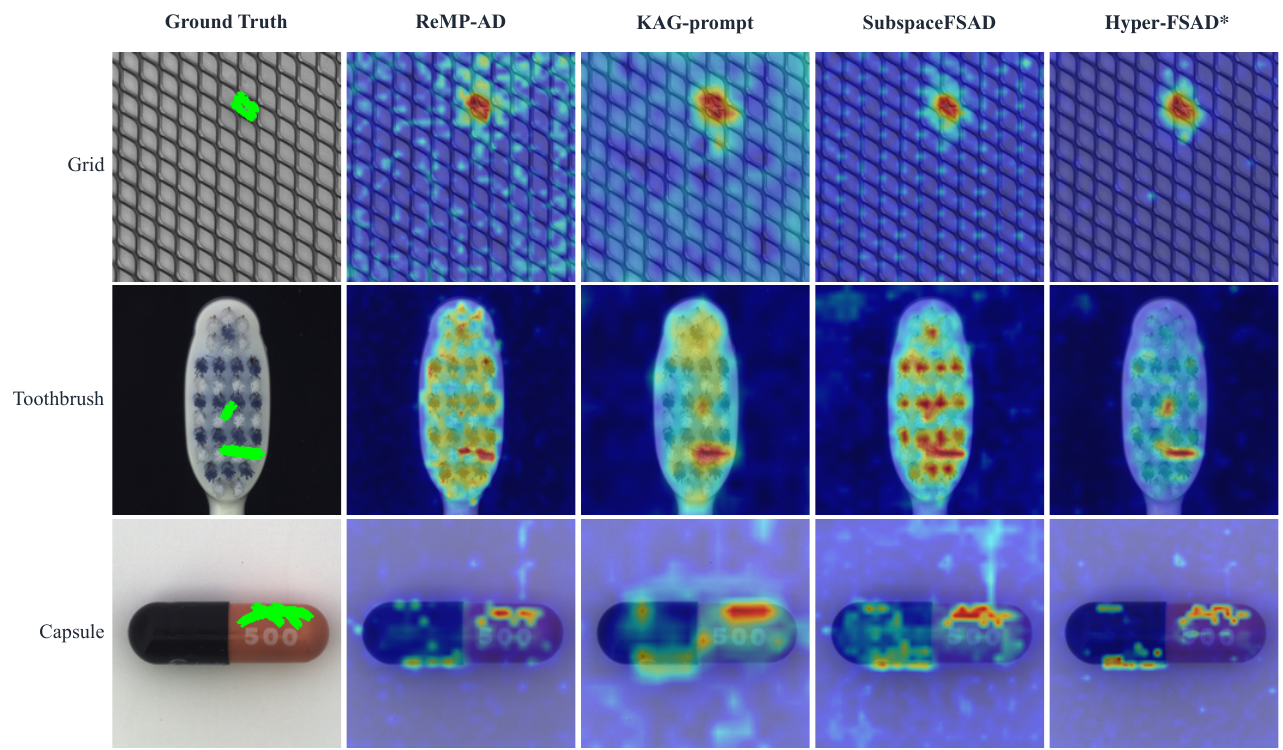}
    \caption{Qualitative comparison of 1-shot anomaly maps.}
    \label{fig:qualitative}
\end{figure}

\subsection{Ablation Study}
Unless otherwise specified, all ablation experiments in this subsection are conducted on MVTecAD~\cite{bergmann2019mvtec} and VisA~\cite{zou2022spot} under the 1/2/4-shot settings, using DINOv3 ViT-B/16 with $448\times448$ inputs. The input-resolution ablation instead varies the resolution from $224\times224$ to $672\times672$ while retaining the same backbone and few-shot protocol.

\subsubsection{The effects of Lookup Strategies.}
\cref{tab:lookup_ablation} compares sparse weighting variants on MVTecAD and VisA; results on MPDD, BTAD, RESC, and BraTS are provided in Appendix~\ref{sec:cross_dataset_lookup_ablation}. Sparsemax (our \emph{Sparse Hyper Matching}) achieves the strongest pixel-level AUROC across both datasets and all shots, together with the best or tied-best image-level results on MVTecAD and competitive results on VisA. Its rows are taken directly from the default DINOv3 ViT-B/16, $448\times448$ configuration in \cref{tab:auroc_results}. Softmax degrades sharply because dense weights mix abundant distractors into reconstruction; entmax provides only partial sparsity, while Top-$p$ remains sensitive to a manually chosen retention ratio. Sparsemax instead forms an adaptive active set with exact zeros, yielding robust localization without a sparsity hyperparameter.

\begin{table}[!t]
\centering
\caption{Patch-weighting ablation on MVTecAD (M) and VisA (V), using DINOv3 ViT-B/16 at $448\times448$. I/P denote I/P-AUROC (\%); sparsemax matches \cref{tab:auroc_results}.}
\label{tab:lookup_ablation}
\setlength{\tabcolsep}{3pt}
\renewcommand{\arraystretch}{0.95}
\small
\vspace{-1mm}
\resizebox{0.98\linewidth}{!}{%
\begin{tabular}{l|cc|cc|cc|cc|cc|cc}
\toprule
\multirow{2}{*}{Patch Weight}
& \multicolumn{2}{c|}{M1} & \multicolumn{2}{c|}{M2} & \multicolumn{2}{c|}{M4}
& \multicolumn{2}{c|}{V1} & \multicolumn{2}{c|}{V2} & \multicolumn{2}{c}{V4} \\
& I & P & I & P & I & P & I & P & I & P & I & P \\
\midrule
Top-$1$ & 93.4 & 95.7 & 94.6 & 95.9 & 95.1 & 96.4 & 92.2 & 96.7 & 92.5 & 96.9 & 92.8 & 96.9 \\
Top-$p$ $(p{=}1\%)$ & 92.8 & 94.3 & 93.6 & 95.1 & 94.1 & 95.4 & 90.1 & 96.6 & 91.7 & 97.3 & 91.8 & 97.8 \\
Softmax & 83.1 & 86.5 & 82.9 & 86.4 & 82.7 & 86.4 & 84.8 & 93.2 & 84.8 & 93.2 & 84.6 & 93.1 \\
Entmax $(\alpha{=}1.5)$ & \textbf{96.3} & 96.1 & 96.6 & 96.6 & 96.7 & 97.0 & \textbf{94.6} & 96.3 & \textbf{95.2} & 97.1 & \textbf{95.5} & 97.5 \\
Sparsemax (ours) & \textbf{96.3} & \textbf{97.7} & \textbf{96.8} & \textbf{97.8} & \textbf{96.9} & \textbf{98.0} & 94.0 & \textbf{97.7} & 94.8 & \textbf{98.0} & \textbf{94.8} & \textbf{98.3} \\
\bottomrule
\end{tabular}%
}
\end{table}

\begin{table}[t]
\centering
\caption{DINOv3 layer-selection analysis on MVTecAD (M)~\cite{bergmann2019mvtec} and VisA (V)~\cite{zou2022spot} at $448\times448$. I/P denote I/P-AUROC (\%); standard deviations are omitted.}
\label{tab:layer_selection_main}
\setlength{\tabcolsep}{3pt}
\renewcommand{\arraystretch}{1.08}
\small
\resizebox{\linewidth}{!}{%
\begin{tabular}{l|cc|cc|cc|cc|cc|cc}
\toprule
\multirow{2}{*}{Layers}
& \multicolumn{2}{c|}{M1} & \multicolumn{2}{c|}{M2} & \multicolumn{2}{c|}{M4}
& \multicolumn{2}{c|}{V1} & \multicolumn{2}{c|}{V2} & \multicolumn{2}{c}{V4} \\
& I & P & I & P & I & P & I & P & I & P & I & P \\
\midrule
$\{6\}$ & 93.5 & 96.2 & 94.0 & 96.6 & 94.4 & 97.0 & 89.7 & 96.2 & 91.1 & 96.7 & 91.2 & 97.1 \\
$\{12\}$ & 92.6 & 95.3 & 93.3 & 95.6 & 94.2 & 95.8 & 84.3 & 93.7 & 87.3 & 95.3 & 88.6 & 95.7 \\
$\{6,12\}$ & 94.4 & 97.4 & 95.2 & 97.5 & 95.5 & 97.7 & 91.2 & 97.3 & 92.9 & 97.9 & 92.9 & 98.1 \\
$\{3,6,9,12\}$ & \textbf{96.3} & \textbf{97.7} & \textbf{96.8} & \textbf{97.8} & \textbf{96.9} & \textbf{98.0} & \textbf{94.0} & \textbf{97.7} & \textbf{94.8} & \textbf{98.0} & \textbf{94.8} & \textbf{98.3} \\
\bottomrule
\end{tabular}%
}
\end{table}

\subsubsection{The effects of DINOv3 layer selection.}
\cref{tab:layer_selection_main} reports the layer-selection analysis on MVTecAD and VisA. A single intermediate or final layer is consistently inferior to combining multiple depths, while $\{6,12\}$ already recovers much of the performance. The default $\{3,6,9,12\}$ configuration gives the strongest overall localization and competitive image-level detection. This behavior is consistent with the complementary-layer ranking guarantee in Theorem~\ref{thm:layer_complementarity}, proved in Appendix~\ref{sec:layer_complementarity}: when layer-wise margins are positive on average and their fluctuations are not perfectly correlated, multi-layer aggregation yields a strictly tighter pairwise-misranking, and hence AUROC, certificate than a single layer. Results on MPDD, BTAD, RESC, and BraTS appear in Appendix~\ref{sec:cross_dataset_layer_selection}.

\subsubsection{The effects of input resolution.}
\label{sec:resolution_analysis}
We study the performance--efficiency trade-off across five input resolutions from $224\times224$ to $672\times672$ under the 1/2/4-shot settings. As shown in \cref{tab:resolution_ablation}, higher resolutions generally benefit both datasets, especially in P-AUROC, because finer spatial details are retained for patch-level matching. The improvement is not strictly monotonic, however, and largely saturates beyond $448\times448$. Under the 4-shot setting, increasing the resolution from $448\times448$ to $672\times672$ changes MVTecAD I/P-AUROC from 96.9\%/98.0\% to 96.7\%/98.1\% and VisA I/P-AUROC from 94.8\%/98.3\% to 94.8\%/98.5\%. Thus, the higher resolution yields only 0.1 and 0.2 percentage-point gains in P-AUROC, respectively, while MVTecAD I-AUROC decreases by 0.2 points and VisA I-AUROC remains unchanged. Meanwhile, runtime increases more than fourfold, from 59.3\,ms to 269.2\,ms. Similar saturation is observed in the 1- and 2-shot settings. Because runtime is determined by the shared backbone, resolution, and inference pipeline, the same values apply to both datasets. We therefore adopt $448\times448$ as the default, which preserves strong image- and pixel-level performance while avoiding the substantial overhead of higher resolutions.

\begin{table}[t]
\centering
\caption{Input-resolution ablation on MVTecAD (M) and VisA (V). I/P: I/P-AUROC (\%); T: shared runtime (ms). Best/second-best values are bolded/underlined.}
\label{tab:resolution_ablation}
\setlength{\tabcolsep}{4pt}
\renewcommand{\arraystretch}{0.82}
\small
\vspace{-1mm}
\resizebox{0.94\linewidth}{!}{%
\begin{tabular}{cl|cc|cc|c}
\toprule
Shot & Resolution & \multicolumn{2}{c|}{M} & \multicolumn{2}{c|}{V} & T \\
& & I & P & I & P & \\
\midrule
& $224\times224$ & 95.4 & 97.1 & 93.1 & 97.4 & \textbf{38.0} \\
& $336\times336$ & 96.0 & 97.4 & 93.5 & 97.7 & \underline{44.2} \\
1-shot & $448\times448$ & 96.3 & \textbf{97.7} & 94.0 & 97.7 & 52.6 \\
& $560\times560$ & \underline{96.6} & 97.5 & \underline{94.3} & \underline{98.0} & 88.6 \\
& $672\times672$ & \textbf{96.7} & \underline{97.6} & \textbf{94.5} & \textbf{98.2} & 165.8 \\
\midrule
& $224\times224$ & 95.8 & 97.4 & 94.0 & 97.8 & \textbf{38.2} \\
& $336\times336$ & \underline{96.3} & 97.7 & 94.4 & 97.8 & \underline{44.3} \\
2-shot & $448\times448$ & \textbf{96.8} & \underline{97.8} & \underline{94.8} & 98.0 & 53.1 \\
& $560\times560$ & \textbf{96.8} & \underline{97.8} & \underline{94.8} & \underline{98.2} & 91.4 \\
& $672\times672$ & \textbf{96.8} & \textbf{97.9} & \textbf{94.9} & \textbf{98.3} & 195.1 \\
\midrule
& $224\times224$ & 96.0 & 97.6 & 93.9 & 97.9 & \textbf{38.4} \\
& $336\times336$ & 96.4 & \underline{98.0} & 94.6 & 98.1 & \underline{44.6} \\
4-shot & $448\times448$ & \textbf{96.9} & \underline{98.0} & \underline{94.8} & \underline{98.3} & 59.3 \\
& $560\times560$ & \underline{96.7} & \textbf{98.1} & \textbf{94.9} & \textbf{98.5} & 106.4 \\
& $672\times672$ & \underline{96.7} & \textbf{98.1} & \underline{94.8} & \textbf{98.5} & 269.2 \\
\bottomrule
\end{tabular}%
}
\end{table}

\subsection{Efficiency Analysis}
\label{sec:efficiency}

We compare Hyper-FSAD with representative FSAD baselines in inference efficiency. \cref{tab:efficiency_comparison} reports GPU memory and per-image inference time under the 1-shot setting on a single NVIDIA RTX 3090 GPU with 24\,GB memory. Although several baselines use lower input resolutions, Hyper-FSAD achieves the lowest latency with low memory consumption, requiring only 0.89\,GB and 52.6\,ms per image. Despite its higher-resolution visual multi-layer pipeline, this efficiency follows from training-free feed-forward inference, which avoids costly prompt optimization, adapter learning, and iterative reconstruction.

\begin{table}[t]
\centering
\caption{Comparison of 1-shot inference efficiency on an NVIDIA RTX 3090.}
\label{tab:efficiency_comparison}
\setlength{\tabcolsep}{6pt}
\renewcommand{\arraystretch}{1.08}
\small
\resizebox{\linewidth}{!}{%
\begin{tabular}{lcccc}
\toprule
Method & Backbone & Resolution & GPU Cost (GB) & Time (ms) \\
\midrule
RegAD        & ResNet18        & $224 \times 224$ & 3.82  & 1517.1 \\
FastRecon    & Wide-ResNet50   & $224 \times 224$ & \textbf{0.55}  & 106.6 \\
WinCLIP      & ViT-L/14-336    & $336 \times 336$ & 15.70 & 8033.5 \\
APRIL-GAN    & ViT-L/14-336    & $336 \times 336$ & 2.05  & 204.6 \\
PromptAD     & ViT-L/14-336    & $336 \times 336$ & 2.10  & 80.1 \\
AnomalyGPT   & ImageBind-Huge  & $224 \times 224$ & 19.17 & 1653.2 \\
ReMP-AD      & ViT-L/14-336    & $336 \times 336$ & 1.48  & 118.8 \\
KAG-prompt   & ViT-L/14-336    & $336 \times 336$ & 3.37  & 100.6 \\
SubspaceAD   & DINOv2 ViT-B/14 & $448 \times 448$ & 1.01  & 86.9 \\
Hyper-FSAD    & DINOv3 ViT-B/16 & $448 \times 448$ & 0.89  & \textbf{52.6} \\
\bottomrule
\end{tabular}%
}
\end{table}

\section{Conclusion}
\label{sec:conclusion}
We presented \textbf{Hyper-FSAD}, a training- and language-free FSAD framework using a frozen visual encoder. Sparse Hyper Matching rejects distractors without fixed top-$p$, while Dual-Branch Image Scoring fuses spatial and global visual evidence. We prove exact below-threshold rejection and conditional few-shot separation. Across six industrial and medical benchmarks and three backbones, Hyper-FSAD delivers consistently strong results without task-specific parameter learning or prompts.

\section*{Acknowledgements}
This work was supported by the National Natural Science Foundation of China (No. 62576176). The computational resources were supported by the Supercomputing Center of Nankai University (NKSC).

\bibliography{main}

@String(CVPR  = {CVPR})

@String(ICCV  = {ICCV})

@String(ECCV  = {ECCV})

@String(ICML  = {ICML})

@String(ICLR  = {ICLR})

@String(AAAI  = {AAAI})

@String(CVPR  = {IEEE Conf. Comput. Vis. Pattern Recog.})

@String(ICCV  = {Int. Conf. Comput. Vis.})

@String(ECCV  = {Eur. Conf. Comput. Vis.})

@String(ICML  = {Int. Conf. Mach. Learn.})

@String(ICLR  = {Int. Conf. Learn. Represent.})

@String(ACL   = {Assoc. Comput. Linguistics})

@inproceedings{mahapatra2021medical,
  title     = {Medical image classification using generalized zero shot learning},
  author    = {Mahapatra, Dwarikanath and Bozorgtabar, Behzad and Ge, Zongyuan},
  booktitle = ICCV,
  pages     = {3344--3353},
  year      = {2021}
}

@article{menze2014multimodal,
  title     = {The multimodal brain tumor image segmentation benchmark (BRATS)},
  author    = {Menze, Bjoern H and Jakab, Andras and Bauer, Stefan and Kalpathy-Cramer, Jayashree and Farahani, Keyvan and Kirby, Justin and Burren, Yuliya and Porz, Nicole and Slotboom, Johannes and Wiest, Roland and others},
  journal   = {IEEE Trans. Med. Imaging},
  volume    = {34},
  number    = {10},
  pages     = {1993--2024},
  year      = {2014}
}

@article{qu2023investigating,
  title     = {Investigating shift equivalence of convolutional neural networks in industrial defect segmentation},
  author    = {Qu, Zhen and Tao, Xian and Shen, Fei and Zhang, Zhengtao and Li, Tao},
  journal   = {IEEE Trans. Instrumentation and Measurement},
  volume    = {72},
  pages     = {1--17},
  year      = {2023}
}

@article{chen2025center,
  title={Center-aware residual anomaly synthesis for multiclass industrial anomaly detection},
  author={Chen, Qiyu and Luo, Huiyuan and Yao, Haiming and Luo, Wei and Qu, Zhen and Lv, Chengkan and Zhang, Zhengtao},
  journal   = {IEEE Trans. Ind. Informatics},
  year={2025},
  publisher={IEEE}
}

@inproceedings{huang2022registration,
  title        = {Registration based few-shot anomaly detection},
  author       = {Huang, Chaoqin and Guan, Haoyan and Jiang, Aofan and Zhang, Ya and Spratling, Michael and Wang, Yan-Feng},
  booktitle    = ECCV,
  pages        = {303--319},
  year         = {2022}
}

@inproceedings{jeong2023winclip,
  title     = {Winclip: Zero-/few-shot anomaly classification and segmentation},
  author    = {Jeong, Jongheon and Zou, Yang and Kim, Taewan and Zhang, Dongqing and Ravichandran, Avinash and Dabeer, Onkar},
  booktitle = CVPR,
  pages     = {19606--19616},
  year      = {2023}
}

@inproceedings{fang2023fastrecon,
  title     = {Fastrecon: Few-shot industrial anomaly detection via fast feature reconstruction},
  author    = {Fang, Zheng and Wang, Xiaoyang and Li, Haocheng and Liu, Jiejie and Hu, Qiugui and Xiao, Jimin},
  booktitle = ICCV,
  pages     = {17481--17490},
  year      = {2023}
}

@inproceedings{li2024promptad,
  title     = {Promptad: Learning prompts with only normal samples for few-shot anomaly detection},
  author    = {Li, Xiaofan and Zhang, Zhizhong and Tan, Xin and Chen, Chengwei and Qu, Yanyun and Xie, Yuan and Ma, Lizhuang},
  booktitle = CVPR,
  pages     = {16838--16848},
  year      = {2024}
}

@inproceedings{zhou2024anomalyclip,
  title     = {AnomalyCLIP: Object-agnostic Prompt Learning for Zero-shot Anomaly Detection},
  author    = {Zhou, Qihang and Pang, Guansong and Tian, Yu and He, Shibo and Chen, Jiming},
  booktitle = ICLR,
  year      = {2024}
}

@inproceedings{ma2025remp,
  title     = {ReMP-AD: Retrieval-enhanced Multi-modal Prompt Fusion for Few-Shot Industrial Visual Anomaly Detection},
  author    = {Ma, Hongchi and Yang, Guanglei and Zhao, Debin and Ji, Yanli and Zuo, Wangmeng},
  booktitle = ICCV,
  pages     = {20425--20434},
  year      = {2025}
}

@inproceedings{tao2025kernel,
  title     = {Kernel-aware graph prompt learning for few-shot anomaly detection},
  author    = {Tao, Fenfang and Xie, Guo-Sen and Zhao, Fang and Shu, Xiangbo},
  booktitle = AAAI,
  pages     = {7347--7355},
  year      = {2025}
}

@inproceedings{lendering2026subspacead,
  title={SubspaceAD: Training-Free Few-Shot Anomaly Detection via Subspace Modeling},
  author={Lendering, Camile and Akdag, Erkut and Bondarev, Egor},
  booktitle=CVPR,
  pages={28557--28566},
  year={2026}
}

@inproceedings{peters2019sparse,
  title     = {Sparse sequence-to-sequence models},
  author    = {Peters, Ben and Niculae, Vlad and Martins, Andr{\'e} FT},
  booktitle = ACL,
  pages     = {1504--1519},
  year      = {2019}
}

@article{simeoni2025dinov3,
  title   = {Dinov3},
  author  = {Sim{\'e}oni, Oriane and Vo, Huy V and Seitzer, Maximilian and Baldassarre, Federico and Oquab, Maxime and Jose, Cijo and Khalidov, Vasil and Szafraniec, Marc and Yi, Seungeun and Ramamonjisoa, Micha{\"e}l and others},
  journal = {arXiv preprint arXiv:2508.10104},
  year    = {2025}
}

@article{guo2024impact,
  title   = {The impact of scanner domain shift on deep learning performance in medical imaging: an experimental study},
  author  = {Guo, Brian and Lu, Darui and Szumel, Gregory and Gui, Rongze and Wang, Tingyu and Konz, Nicholas and Mazurowski, Maciej A},
  journal = {arXiv preprint arXiv:2409.04368},
  year    = {2024}
}

@article{jsss-14-119-2025,
  author  = {Goodarzi, P. and Sch\"utze, A. and Schneider, T.},
  title   = {Domain shifts in industrial condition monitoring: a comparative analysis of automated machine learning models},
  journal = {Journal of Sensors and Sensor Systems},
  volume  = {14},
  year    = {2025},
  number  = {2},
  pages   = {119--132}
}

@article{wang2025industrial,
  title     = {Industrial vision inspection using digital twins: bridging CAD models and realistic scenarios},
  author    = {Wang, Fangjun and Wu, Jianhao and Yang, Zhouwang and Song, Yanzhi},
  journal   = {Journal of Intelligent Manufacturing},
  volume    = {36},
  number    = {7},
  pages     = {4963--4975},
  year      = {2025}
}

@article{zhang2026training,
  title={Is Task-Specific Training Necessary for Anomaly Detection?},
  author={Zhang, Xingwu and Li, Guanxuan and Henderson, Paul and Aragon-Camarasa, Gerardo and Long, Zijun},
  journal={arXiv preprint arXiv:2601.22763},
  year={2026}
}

@inproceedings{cao2024adaclip,
  title={Adaclip: Adapting clip with hybrid learnable prompts for zero-shot anomaly detection},
  author={Cao, Yunkang and Zhang, Jiangning and Frittoli, Luca and Cheng, Yuqi and Shen, Weiming and Boracchi, Giacomo},
  booktitle=ECCV,
  pages={55--72},
  year={2024}
}

@inproceedings{gu2024anomalygpt,
  title={Anomalygpt: Detecting industrial anomalies using large vision-language models},
  author={Gu, Zhaopeng and Zhu, Bingke and Zhu, Guibo and Chen, Yingying and Tang, Ming and Wang, Jinqiao},
  booktitle=AAAI,
  pages={1932--1940},
  year={2024}
}

@inproceedings{qu2024vcp,
  title={Vcp-clip: A visual context prompting model for zero-shot anomaly segmentation},
  author={Qu, Zhen and Tao, Xian and Prasad, Mukesh and Shen, Fei and Zhang, Zhengtao and Gong, Xinyi and Ding, Guiguang},
  booktitle=ECCV,
  pages={301--317},
  year={2024}
}

@inproceedings{martins2016softmax,
  title={From softmax to sparsemax: A sparse model of attention and multi-label classification},
  author={Martins, Andre and Astudillo, Ramon},
  booktitle=ICML,
  pages={1614--1623},
  year={2016}
}

@inproceedings{bergmann2019mvtec,
  title={MVTec AD--A comprehensive real-world dataset for unsupervised anomaly detection},
  author={Bergmann, Paul and Fauser, Michael and Sattlegger, David and Steger, Carsten},
  booktitle=CVPR,
  pages={9592--9600},
  year={2019}
}

@inproceedings{zou2022spot,
  title={Spot-the-difference self-supervised pre-training for anomaly detection and segmentation},
  author={Zou, Yang and Jeong, Jongheon and Pemula, Latha and Zhang, Dongqing and Dabeer, Onkar},
  booktitle=ECCV,
  pages={392--408},
  year={2022}
}

@inproceedings{jezek2021deep,
  title={Deep learning-based defect detection of metal parts: evaluating current methods in complex conditions},
  author={Jezek, Stepan and Jonak, Martin and Burget, Radim and Dvorak, Pavel and Skotak, Milos},
  booktitle={2021 13th International congress on ultra modern telecommunications and control systems and workshops (ICUMT)},
  pages={66--71},
  year={2021},
  organization={IEEE}
}

@inproceedings{mishra2021vt,
  title={VT-ADL: A Vision Transformer Network for Image Anomaly Detection and Localization},
  author={Mishra, P and Verk, R and Fornasier, D and Piciarelli, C and Foresti, GL and others},
  booktitle={IEEE International Symposium on Industrial Electronics},
  pages={01--06},
  year={2021}
}

@article{hu2019automated,
  title={Automated segmentation of macular edema in OCT using deep neural networks},
  author={Hu, Junjie and Chen, Yuanyuan and Yi, Zhang},
  journal={Medical image analysis},
  volume={55},
  pages={216--227},
  year={2019}
}

@article{chen2023april,
  title={April-gan: A zero-/few-shot anomaly classification and segmentation method for cvpr 2023 vand workshop challenge tracks 1\&2: 1st place on zero-shot ad and 4th place on few-shot ad},
  author={Chen, Xuhai and Han, Yue and Zhang, Jiangning},
  journal={arXiv preprint arXiv:2305.17382},
  year={2023}
}

@inproceedings{radford2021learning,
  title={Learning transferable visual models from natural language supervision},
  author={Radford, Alec and Kim, Jong Wook and Hallacy, Chris and Ramesh, Aditya and Goh, Gabriel and Agarwal, Sandhini and Sastry, Girish and Askell, Amanda and Mishkin, Pamela and Clark, Jack and others},
  booktitle=ICML,
  pages={8748--8763},
  year={2021}
}

@article{oquab2023dinov2,
  title={Dinov2: Learning robust visual features without supervision},
  author={Oquab, Maxime and Darcet, Timoth{\'e}e and Moutakanni, Th{\'e}o and Vo, Huy and Szafraniec, Marc and Khalidov, Vasil and Fernandez, Pierre and Haziza, Daniel and Massa, Francisco and El-Nouby, Alaaeldin and others},
  journal={arXiv preprint arXiv:2304.07193},
  year={2023}
}

\clearpage
\appendix
\noindent\rule{\linewidth}{0.5pt}
\section*{Appendix Overview}
This appendix provides supplementary details supporting the main manuscript, organized as follows:

\begin{itemize}
    \setlength{\itemsep}{2pt}
    \setlength{\parskip}{0pt}
    \setlength{\parsep}{0pt}
    \item[--] \textbf{Sec.~A} provides proofs and additional theoretical details.
    \item[--] \textbf{Sec.~B} reports additional metrics and cross-dataset patch-weighting and layer-selection ablations.
    \item[--] \textbf{Sec.~C} studies the effect of the fusion strategy.
    \item[--] \textbf{Sec.~D} studies the effect of the pooling operator.
    \item[--] \textbf{Sec.~E} provides additional qualitative results.
\end{itemize}

\setcounter{secnumdepth}{2}
\section{Proofs and Additional Theoretical Details}
\label{sec:theory_proofs}

\subsection{Proof of Proposition~\ref{prop:sparse_support}}
For clarity, we use column vectors in the proofs and omit the layer index. Let $\Delta_M=\{\mathbf w\in\mathbb R^M:\mathbf w\geq\mathbf0,\mathbf1^\top\mathbf w=1\}$ and $\Pi_{\Delta_M}$ denote Euclidean projection onto this simplex.

\begin{proof}
The KKT conditions for
$\min_{\mathbf w\in\Delta_M}\frac12\|\mathbf w-\mathbf z\|_2^2$
give the unique solution
\begin{equation}
w_i=[z_i-\tau]_+,
\qquad \sum_{i=1}^{M}[z_i-\tau]_+=1,
\label{eq:kkt_sparsemax}
\end{equation}
where $[a]_+=\max(a,0)$. Under \cref{eq:support_condition}, setting $\tau=\tau_R$ makes the active coordinates exactly $R$, and
$\sum_{i\in R}(z_i-\tau_R)=1$ by definition. This proves the support and weight formula.

Now append similarities $\tilde z_j\leq\tau_R$. The same threshold satisfies
\begin{equation}
\sum_{i=1}^{M}[z_i-\tau_R]_+
+\sum_j[\tilde z_j-\tau_R]_+=1+0=1.
\end{equation}
Hence \cref{eq:kkt_sparsemax} gives the extended solution
$[\mathbf w;\mathbf0]$. Since the simplex projection is unique, neither the original weights nor the reconstruction $\mathbf P^\top\mathbf w$ changes.

It remains to prove stability. A projection onto any nonempty closed convex set is non-expansive. In particular, the variational inequalities for
$\mathbf w=\Pi_{\Delta_M}(\mathbf z)$ and
$\mathbf w'=\Pi_{\Delta_M}(\mathbf z')$ imply
\begin{equation}
\|\mathbf w-\mathbf w'\|_2^2
\leq \langle\mathbf w-\mathbf w',\mathbf z-\mathbf z'\rangle
\leq \|\mathbf w-\mathbf w'\|_2\|\mathbf z-\mathbf z'\|_2.
\end{equation}
Canceling the common factor proves the non-expansiveness claim. Since
$\mathbf z=\mathbf P\mathbf q$ and
$\boldsymbol\varepsilon(\mathbf q)=\mathbf P^\top\Pi_{\Delta_M}(\mathbf P\mathbf q)$,
submultiplicativity further gives the reconstruction bound
\begin{align}
\|\boldsymbol\varepsilon(\mathbf q)-\boldsymbol\varepsilon(\mathbf q')\|_2
&\leq \|\mathbf P\|_2
\|\Pi_{\Delta_M}(\mathbf P\mathbf q)-\Pi_{\Delta_M}(\mathbf P\mathbf q')\|_2\\
&\leq\|\mathbf P\|_2^2\|\mathbf q-\mathbf q'\|_2,
\label{eq:reconstruction_stability}
\end{align}
which holds even when the active set changes.
\end{proof}

The invariance is specific to sufficiently low-scoring additions. For comparison, if $J$ new patches all have similarity $a$, dense softmax assigns them total mass
\begin{equation}
\frac{J e^a}{\sum_{i=1}^{M}e^{z_i}+J e^a}\longrightarrow1
\quad\text{as }J\to\infty,
\end{equation}
whereas sparsemax remains exactly unchanged whenever $a\leq\tau_R$.

\paragraph{Local form on a fixed active set.}
When the active set $R$ does not change, let
$\mathbf H_R=\mathbf I-|R|^{-1}\mathbf1\mathbf1^\top$ and let $\mathbf P_R$ contain its selected support rows. Direct substitution of the threshold gives
\begin{align}
\boldsymbol\varepsilon(\mathbf q)
&=\bar{\mathbf p}_R+\mathbf G_R\mathbf q,\\
\mathbf G_R
&=\mathbf P_R^\top\mathbf H_R\mathbf P_R\\
&=\sum_{i\in R}(\mathbf p_i-\bar{\mathbf p}_R)
(\mathbf p_i-\bar{\mathbf p}_R)^\top,
\label{eq:local_sparse_gain}
\end{align}
where $\bar{\mathbf p}_R=|R|^{-1}\sum_{i\in R}\mathbf p_i$. Thus the local reconstruction gain is governed by the dispersion $\|\mathbf G_R\|_2$ of the \emph{selected} normal evidence, rather than directly by the total memory size.

\subsection{Proof of Theorem~\ref{thm:separation}}
\begin{proof}
All tokens are unit-normalized. For a normal query patch and every active support patch,
\begin{equation}
\langle\mathbf q_N^l,\mathbf p_i^l\rangle
=1-\tfrac12\|\mathbf q_N^l-\mathbf p_i^l\|_2^2
\geq1-\tfrac12\epsilon_l^2.
\end{equation}
Because the generic retrieval weights lie in $\Delta_M$, the reconstruction is a convex combination:
$\boldsymbol\varepsilon_N^l=\sum_i w_i^l\mathbf p_i^l$. Therefore,
\begin{equation}
\langle\mathbf q_N^l,\boldsymbol\varepsilon_N^l\rangle
\geq1-\tfrac12\epsilon_l^2,
\qquad
\|\boldsymbol\varepsilon_N^l\|_2\leq1.
\end{equation}
The first quantity is positive because $\epsilon_l<\sqrt2$, and hence
$\mathrm{cos}(\mathbf q_N^l,\boldsymbol\varepsilon_N^l)
\geq1-\epsilon_l^2/2$. Every patch of $N$ consequently has layer-averaged score at most
$|L|^{-1}\sum_l\epsilon_l^2/2$, so max pooling preserves this upper bound.

For the specified anomalous patch, each reconstruction lies in
$\operatorname{conv}(\mathcal P^l)$. Assumption (ii) of Theorem~\ref{thm:separation} therefore gives a layer score of at least $\gamma_l$. Averaging across layers and then taking the maximum over patches gives
\begin{equation}
S_{\mathrm{map}}(A)-S_{\mathrm{map}}(N)
\geq\frac1{|L|}\sum_{l\in L}\left(\gamma_l-\frac{\epsilon_l^2}{2}\right)
=\Delta_{\mathrm{map}}.
\end{equation}

For unit <CLS> tokens,
\begin{equation}
1-\langle\mathbf c,\mathbf c_k\rangle
=\tfrac12\|\mathbf c-\mathbf c_k\|_2^2.
\label{eq:cosine_distance_identity}
\end{equation}
The two <CLS> distance assumptions thus imply
\begin{equation}
S_{\mathrm{cls}}(A)-S_{\mathrm{cls}}(N)
\geq\frac1{2|L|}\sum_{l\in L}(\delta_l^2-r_l^2)
=\Delta_{\mathrm{cls}}.
\end{equation}
Applying the convex fusion in \cref{eq:final_score} proves \cref{eq:fused_margin}. If its right-hand side is positive, the anomalous score is strictly larger than the normal score, so a threshold between them separates the pair. If the assumptions hold uniformly for two populations, the same threshold separates those populations.
\end{proof}

\subsection{Complementarity of Multi-Layer Evidence}
\label{sec:layer_complementarity}
For compactness, define the layer-wise patch and global discrepancies
\begin{equation}
d_{X,t}^l=1-\mathrm{cos}(\mathbf q_{X,t}^l,\boldsymbol\varepsilon_{X,t}^l),
\qquad
g_X^l=1-\max_k\langle\mathbf c_X^l,\mathbf c_k^l\rangle.
\label{eq:layer_discrepancies}
\end{equation}
We write $S_{\mathrm{image}}^L$ for the score in \cref{eq:final_score} computed with layer set $L$.

\begin{theorem}[Complementary-layer ranking guarantee]
\label{thm:layer_complementarity}
Let $m=|L|>1$ and consider a normal--anomalous pair $(N,A)$ with a common anomalous patch $t^\star$ across the selected feature grids. Define the layer-wise ranking certificate
\begin{equation}
\eta_l=\lambda\!\left(d_{A,t^\star}^l-\max_t d_{N,t}^l\right)
 +(1-\lambda)(g_A^l-g_N^l).
\label{eq:layer_margin_certificate}
\end{equation}
Then
\begin{equation}
S_{\mathrm{image}}^L(A)-S_{\mathrm{image}}^L(N)
\geq \bar\eta_L:=\frac1m\sum_{l\in L}\eta_l.
\label{eq:multilayer_margin_certificate}
\end{equation}
Under the assumptions of Theorem~\ref{thm:separation}, each certificate further satisfies
\begin{equation}
\eta_l\geq b_l:=\lambda\!\left(\gamma_l-\frac{\epsilon_l^2}{2}\right)
 +\frac{1-\lambda}{2}(\delta_l^2-r_l^2).
\label{eq:layer_margin_lower_bound}
\end{equation}
Thus $\sum_l b_l>0$ certifies multi-layer separation, even if some individual $b_l$ are non-positive.

Now let the support set and normal--anomalous pair be random. Suppose that, for every $l\in L$,
$\mathbb E[\eta_l]\geq\mu>0$,
$\operatorname{Var}(\eta_l)\leq\sigma^2$, and
$\operatorname{Cov}(\eta_l,\eta_r)\leq\rho\sigma^2$ for $l\neq r$, where $0\leq\rho<1$. Then
\begin{equation}
\Pr\!\left[S_{\mathrm{image}}^L(A)\leq S_{\mathrm{image}}^L(N)\right]
\leq
\frac{\sigma^2 c_m}{\sigma^2 c_m+\mu^2},
\label{eq:correlated_layer_bound}
\end{equation}
where $c_m=[1+(m-1)\rho]/m<1$.
For $\sigma>0$, this is strictly tighter than the corresponding uniform single-layer certificate $\sigma^2/(\sigma^2+\mu^2)$.
\end{theorem}

\begin{proof}
Using \cref{eq:generic_map_score}, the order of averaging and max pooling gives
\begin{align}
S_{\mathrm{map}}^L(A)
&\geq \frac1m\sum_{l\in L}d_{A,t^\star}^l,
\label{eq:anomaly_map_lower}\\
S_{\mathrm{map}}^L(N)
&=\max_t\frac1m\sum_{l\in L}d_{N,t}^l
 \leq\frac1m\sum_{l\in L}\max_t d_{N,t}^l.
\label{eq:normal_map_upper}
\end{align}
Moreover, \cref{eq:cls_score} gives
$S_{\mathrm{cls}}^L(A)-S_{\mathrm{cls}}^L(N)
=m^{-1}\sum_l(g_A^l-g_N^l)$.
Applying the convex fusion in \cref{eq:final_score} proves \cref{eq:multilayer_margin_certificate}. The patch and global bounds established in the proof of Theorem~\ref{thm:separation} give
\begin{align}
d_{A,t^\star}^l-\max_t d_{N,t}^l
&\geq\gamma_l-\epsilon_l^2/2,\nonumber\\
g_A^l-g_N^l&\geq(\delta_l^2-r_l^2)/2,
\end{align}
which proves \cref{eq:layer_margin_lower_bound}.

It remains to establish the ranking guarantee. Let
$\bar\mu_L=\mathbb E[\bar\eta_L]$ and
$v_L=\operatorname{Var}(\bar\eta_L)$.
Because the actual score gap is lower-bounded by $\bar\eta_L$, a misranking implies $\bar\eta_L\leq0$. Cantelli's inequality therefore yields
\begin{equation}
\Pr(\bar\eta_L\leq0)
=\Pr(\bar\eta_L-\bar\mu_L\leq-\bar\mu_L)
\leq\frac{v_L}{v_L+\bar\mu_L^2}.
\label{eq:cantelli_layer_margin}
\end{equation}
Under the stated covariance conditions,
\begin{align}
v_L
&=\frac1{m^2}\!\left(\sum_l\operatorname{Var}(\eta_l)
 +2\!\sum_{l<r}\operatorname{Cov}(\eta_l,\eta_r)\right)\\
&\leq\frac{\sigma^2}{m}\bigl(1+(m-1)\rho\bigr)
=\sigma^2c_m.
\end{align}
Together with $\bar\mu_L\geq\mu$, this proves \cref{eq:correlated_layer_bound}. Since $m>1$ and $\rho<1$, we have $c_m<1$, so its right-hand side is strictly smaller than the uniform single-layer bound whenever $\sigma>0$.
\end{proof}

Under the standard pairwise-ranking interpretation of AUROC, \cref{eq:cantelli_layer_margin} also gives
$\mathrm{AUROC}_L\geq1-v_L/(v_L+\bar\mu_L^2)$ when ties are counted pessimistically. For pre-interpolation pixel scores, the same derivation applies to
$\eta_l^{\mathrm{patch}}=d_{A,t_A}^l-d_{N,t_N}^l$.
The result is conditional rather than a claim that adding an arbitrary layer must help: a strongly negative mean margin can reduce $\bar\mu_L$, while perfectly correlated fluctuations ($\rho=1$) provide no concentration gain. It therefore attributes the improvement to complementary evidence across depths, rather than to layer count alone.

\subsection{Few-Shot Coverage and Stability of the Global Branch}
For a fixed normal token $\mathbf c_N^l$, let $\mathbf c_{\mathrm{sup}}^l$
denote one independently sampled normal support token and define
$p_l=\Pr(\|\mathbf c_{\mathrm{sup}}^l-\mathbf c_N^l\|_2\leq r_l\mid\mathbf c_N^l)\geq\alpha_l$.
Independence of the $K$ sampled normal supports gives
\begin{align}
&\Pr\!\left[\min_k\|\mathbf c_N^l-\mathbf c_k^l\|_2>r_l\right]\\
&\hspace{12mm}=(1-p_l)^K\leq(1-\alpha_l)^K\leq e^{-K\alpha_l},
\end{align}
which proves \cref{eq:fewshot_coverage}. A union bound shows that all selected layers satisfy their respective coverage conditions with probability at least
$1-\sum_{l\in L}e^{-K\alpha_l}$; no independence across layers is required for this final step.

The global score also has a direct stability property. For a fixed unit-normalized memory $\mathcal C=\{\mathbf c_k\}_{k=1}^K$, define
$s_{\mathcal C}(\mathbf c)=1-\max_k\langle\mathbf c,\mathbf c_k\rangle$.
It is zero on every stored token, and for any unit $\mathbf u,\mathbf v$,
\begin{align}
|s_{\mathcal C}(\mathbf u)-s_{\mathcal C}(\mathbf v)|
&=\left|\max_k\langle\mathbf v,\mathbf c_k\rangle-
\max_k\langle\mathbf u,\mathbf c_k\rangle\right|\\
&\leq\max_k|\langle\mathbf u-\mathbf v,\mathbf c_k\rangle|
\leq\|\mathbf u-\mathbf v\|_2.
\end{align}
Thus each layer's implemented cosine score is $1$-Lipschitz on the unit sphere. By \cref{eq:cosine_distance_identity},
$\sqrt{2s_{\mathcal C}(\mathbf c)}=\min_k\|\mathbf c-\mathbf c_k\|_2$ is exactly the distance to the fixed normal <CLS> memory. The maximal stable-score result for distance-to-memory~\cite{zhang2026training} therefore applies to this square-root distance, but not to the sparse patch-reconstruction score.

\subsection{Sparsemax Computation}
\label{sec:sparsemax_algorithm}
\cref{alg:spm} gives the closed-form computation used by Sparse Hyper Matching. The permutation is explicitly inverted before returning the weights so that they remain aligned with the original memory-bank rows.

\begin{algorithm}[ht]
\caption{Sparsemax with Original-Index Recovery}
\label{alg:spm}
\begin{algorithmic}[1]
\Require Similarities $\mathbf z^l\in\mathbb R^M$, where $M=KN_p$
\Ensure Sparse weights $\hat{\mathbf w}^l\in\mathbb R^M$
\State Sort and record $\pi$: $z^l_{\pi(1)}\geq\cdots\geq z^l_{\pi(M)}$
\State $m\gets\max\{j:1+jz^l_{\pi(j)}>\sum_{u=1}^{j}z^l_{\pi(u)}\}$
\State $\tau\gets\big(\sum_{u=1}^{m}z^l_{\pi(u)}-1\big)/m$
\For{$j=1,\ldots,M$}
    \State $\hat w^l_{\pi(j)}\gets\max(0,z^l_{\pi(j)}-\tau)$
\EndFor
\State \Return $\hat{\mathbf w}^l$
\end{algorithmic}
\end{algorithm}

\section{Detailed Results on Additional Metrics}
\label{sec:metrics_detail}

In this section, we report detailed quantitative results on additional evaluation metrics beyond the AUROC results presented in the main paper. Specifically, \cref{tab:if1_iap_results} reports image-level F1-max (I-F1) and average precision (I-AP), while \cref{tab:ppro_pap_results} reports pixel-level PRO (P-PRO) and average precision (P-AP), across six datasets under 1-shot, 2-shot, and 4-shot settings. For consistency with the main paper, all methods are evaluated on the same support sets and the results are averaged over five runs. The best and second-best results are marked in \textbf{bold} and \underline{underlined}, respectively. As shown in \cref{tab:if1_iap_results,tab:ppro_pap_results}, the superiority of Hyper-FSAD extends consistently beyond AUROC to threshold-dependent and region-aware metrics. At the image level, our method achieves the strongest overall performance in both I-F1 and I-AP across different shot settings, indicating that the proposed scoring mechanism not only ranks anomalous samples effectively, but also produces more reliable decision boundaries.

At the pixel level, Hyper-FSAD also delivers highly competitive or best results on P-PRO and P-AP for most datasets and shot settings, demonstrating that Sparse Hyper Matching yields anomaly maps with improved localization quality and region consistency. These observations further support our claim that the combination of sparse reconstruction and dual-branch scoring provides robust benefits for both classification and segmentation in a training-free and language-free setting.

\begin{table*}[p]
\centering
\captionsetup{skip=3pt}
\caption{Image-level F1-max (\textbf{I-F1}) and average precision (\textbf{I-AP}) across six datasets and 1/2/4-shot settings. Symbols, highlighting, backbone variants, and reporting protocol follow \cref{tab:auroc_results}.}
\renewcommand{\arraystretch}{0.77}%
    \resizebox{0.98\textwidth}{!}{%
\begin{tabular}{ll|l|cc|cc|cc|cc|cc|cc}
\toprule
& \multirow{2}{*}{Methods} & \multirow{2}{*}{Venue/Year} &
\multicolumn{2}{c|}{MVTecAD} & \multicolumn{2}{c|}{VisA} & \multicolumn{2}{c|}{MPDD} &
\multicolumn{2}{c|}{BTAD} & \multicolumn{2}{c|}{RESC} & \multicolumn{2}{c}{BraTS} \\
\cmidrule(lr){4-5}\cmidrule(lr){6-7}\cmidrule(lr){8-9}\cmidrule(lr){10-11}\cmidrule(lr){12-13}\cmidrule(lr){14-15}
& & & I-F1 & I-AP & I-F1 & I-AP & I-F1 & I-AP & I-F1 & I-AP & I-F1 & I-AP & I-F1 & I-AP \\
\midrule
\multirow{12}{*}{\rotatebox[origin=c]{90}{1-shot}}
& RegAD & ECCV'22 & 87.1 & 87.2 & 76.2 & 72.2 & 72.9 & 61.5 & 78.2 & 80.5 & 60.4 & 46.6 & 83.0 & 73.2 \\
& FastRecon & ICCV'23 & 93.4 & 95.6 & 81.4 & 82.3 & 80.3 & 76.1 & 92.0 & 95.0 & 71.3 & 80.4 & 86.5 & 85.1 \\
& WinCLIP & CVPR'23 & 92.0 & 96.1 & 82.8 & 87.0 & 80.9 & 82.5 & 81.8 & 86.3 & 60.7 & 48.1 & 87.4 & 92.5 \\
& APRIL-GAN & CVPR'23 & 90.9 & 95.6 & 83.1 & 90.5 & 80.1 & 80.8 & 84.0 & 88.5 & 69.5 & 69.7 & 88.9 & 92.5 \\
& PromptAD & CVPR'24 & 93.7 & 96.6 & 83.3 & 86.8 & 81.6 & 83.5 & 90.4 & 94.4 & 78.2 & 84.3 & 87.5 & 88.5 \\
& AnomalyGPT & AAAI'24 & 94.3 & 96.1 & 84.4 & 87.4 & 79.3 & 75.9 & 89.7 & 94.6 & 76.2 & 83.4 & 85.8 & 82.0 \\
& ReMP-AD & ICCV'25 & 94.3 & \underline{97.3} & 86.9 & 93.1 & 82.9 & \underline{84.0} & 91.5 & 95.4 & 72.6 & 78.2 & 87.5 & 93.8 \\
& KAG-prompt & AAAI'25 & 94.9 & \underline{97.3} & 86.9 & 92.1 & 79.4 & 70.8 & 90.0 & 95.5 & 64.4 & 62.6 & 86.5 & 89.9 \\
& SubspaceAD$^\ddag$ & CVPR'26 & 95.4 & 96.6 & 89.6 & 93.5 & 82.6 & 81.4 & \underline{94.4} & 97.7 & 71.2 & 61.4 & \textbf{89.5} & \textbf{95.5} \\
& Hyper-FSAD$^\dag$ & \multicolumn{1}{c|}{-} & 93.2\textsubscript{$\pm$0.2} & 96.7\textsubscript{$\pm$0.1} & 89.1\textsubscript{$\pm$0.5} & 93.2\textsubscript{$\pm$0.6} & \underline{85.2\textsubscript{$\pm$0.4}} & \underline{84.0\textsubscript{$\pm$2.0}} & 88.4\textsubscript{$\pm$1.4} & 94.0\textsubscript{$\pm$1.2} & \underline{80.7\textsubscript{$\pm$1.9}} & \underline{85.8\textsubscript{$\pm$1.6}} & \underline{89.4\textsubscript{$\pm$0.8}} & \underline{94.9\textsubscript{$\pm$0.5}} \\
& Hyper-FSAD$^\ddag$ & \multicolumn{1}{c|}{-} & \textbf{96.8\textsubscript{$\pm$0.1}} & \textbf{98.1\textsubscript{$\pm$0.0}} & \underline{89.9\textsubscript{$\pm$0.1}} & \underline{93.8\textsubscript{$\pm$0.2}} & \textbf{85.9\textsubscript{$\pm$0.6}} & \textbf{85.8\textsubscript{$\pm$1.0}} & 93.6\textsubscript{$\pm$1.9} & \underline{98.1\textsubscript{$\pm$1.2}} & 77.8\textsubscript{$\pm$0.2} & 71.2\textsubscript{$\pm$0.3} & 86.9\textsubscript{$\pm$1.6} & 87.8\textsubscript{$\pm$1.8} \\
& Hyper-FSAD$^\ast$ & \multicolumn{1}{c|}{-} & \underline{96.6\textsubscript{$\pm$0.3}} & \textbf{98.1\textsubscript{$\pm$0.1}} & \textbf{90.6\textsubscript{$\pm$0.1}} & \textbf{94.0\textsubscript{$\pm$0.1}} & 82.5\textsubscript{$\pm$0.4} & 81.6\textsubscript{$\pm$2.1} & \textbf{95.1\textsubscript{$\pm$0.2}} & \textbf{98.7\textsubscript{$\pm$0.4}} & \textbf{83.1\textsubscript{$\pm$1.4}} & \textbf{90.4\textsubscript{$\pm$2.3}} & 87.7\textsubscript{$\pm$1.0} & 89.6\textsubscript{$\pm$1.1} \\
\midrule
\multirow{12}{*}{\rotatebox[origin=c]{90}{2-shot}}
& RegAD & ECCV'22 & 88.8 & 88.9 & 75.8 & 73.6 & 73.6 & 62.2 & 89.2 & 92.1 & 62.0 & 48.0 & 83.5 & 72.0 \\
& FastRecon & ICCV'23 & 94.5 & 96.5 & 81.8 & 81.3 & 83.0 & 81.4 & 90.3 & 94.9 & 75.7 & 84.9 & 87.2 & 83.6 \\
& WinCLIP & CVPR'23 & 93.0 & 96.6 & 81.3 & 85.9 & 81.0 & 83.3 & 84.2 & 87.6 & 61.0 & 50.6 & 88.0 & 93.4 \\
& APRIL-GAN & CVPR'23 & 91.0 & 95.5 & 82.6 & 90.4 & 79.4 & 80.2 & 84.2 & 88.5 & 70.8 & 71.0 & \underline{89.2} & 93.0 \\
& PromptAD & CVPR'24 & 95.1 & 97.7 & 83.0 & 87.0 & 83.6 & \textbf{88.2} & 89.0 & 94.4 & \underline{79.6} & \underline{85.9} & 88.1 & 89.3 \\
& AnomalyGPT & AAAI'24 & 95.0 & 97.0 & 84.1 & 88.8 & 82.1 & 81.1 & 89.9 & 95.0 & 78.2 & 83.5 & 86.6 & 83.3 \\
& ReMP-AD & ICCV'25 & 94.9 & 97.8 & 89.0 & \underline{94.7} & \underline{85.5} & \underline{86.5} & 92.1 & 96.6 & 76.0 & 81.7 & 88.6 & \underline{94.3} \\
& KAG-prompt & AAAI'25 & 96.2 & \textbf{98.4} & 88.6 & 93.8 & 79.9 & 71.1 & 90.1 & 95.1 & 67.3 & 75.3 & 86.5 & 92.3 \\
& SubspaceAD$^\ddag$ & CVPR'26 & 96.3 & 97.3 & 90.2 & 93.4 & 85.0 & 84.2 & \textbf{95.5} & \underline{98.5} & 71.8 & 58.3 & 89.0 & 93.6 \\
& Hyper-FSAD$^\dag$ & \multicolumn{1}{c|}{-} & 93.7\textsubscript{$\pm$0.3} & 96.8\textsubscript{$\pm$0.4} & 90.4\textsubscript{$\pm$0.4} & \textbf{95.1\textsubscript{$\pm$0.4}} & \textbf{85.9\textsubscript{$\pm$0.4}} & 85.9\textsubscript{$\pm$0.9} & 89.8\textsubscript{$\pm$1.1} & 95.1\textsubscript{$\pm$0.6} & 79.3\textsubscript{$\pm$1.1} & 84.3\textsubscript{$\pm$0.9} & \textbf{89.9\textsubscript{$\pm$1.6}} & \textbf{95.2\textsubscript{$\pm$1.2}} \\
& Hyper-FSAD$^\ddag$ & \multicolumn{1}{c|}{-} & \textbf{97.1\textsubscript{$\pm$0.2}} & 98.0\textsubscript{$\pm$0.1} & \underline{91.1\textsubscript{$\pm$0.2}} & \textbf{95.1\textsubscript{$\pm$0.2}} & 85.1\textsubscript{$\pm$0.7} & 84.6\textsubscript{$\pm$0.7} & 93.9\textsubscript{$\pm$2.0} & 98.0\textsubscript{$\pm$0.9} & 78.9\textsubscript{$\pm$0.2} & 70.0\textsubscript{$\pm$0.3} & 86.8\textsubscript{$\pm$1.9} & 86.2\textsubscript{$\pm$1.0} \\
& Hyper-FSAD$^\ast$ & \multicolumn{1}{c|}{-} & \underline{96.9\textsubscript{$\pm$0.2}} & \underline{98.2\textsubscript{$\pm$0.0}} & \textbf{91.7\textsubscript{$\pm$0.2}} & 94.5\textsubscript{$\pm$0.3} & 82.5\textsubscript{$\pm$0.1} & 83.6\textsubscript{$\pm$0.4} & \underline{95.3\textsubscript{$\pm$0.7}} & \textbf{98.6\textsubscript{$\pm$0.3}} & \textbf{81.9\textsubscript{$\pm$2.2}} & \textbf{87.3\textsubscript{$\pm$1.2}} & 89.0\textsubscript{$\pm$2.1} & 89.2\textsubscript{$\pm$1.7} \\
\midrule
\multirow{12}{*}{\rotatebox[origin=c]{90}{4-shot}}
& RegAD & ECCV'22 & 89.8 & 91.7 & 77.1 & 73.9 & 75.9 & 66.9 & 91.2 & 94.5 & 63.7 & 51.0 & 83.9 & 75.5 \\
& FastRecon & ICCV'23 & 95.3 & 97.2 & 82.6 & 85.1 & 82.8 & 81.0 & 91.8 & 96.2 & 76.4 & 84.6 & 87.6 & 86.4 \\
& WinCLIP & CVPR'23 & 94.0 & 97.3 & 82.8 & 87.8 & 83.1 & 86.1 & 84.3 & 88.0 & 62.6 & 54.0 & 88.0 & 93.4 \\
& APRIL-GAN & CVPR'23 & 91.6 & 95.9 & 83.3 & 91.1 & 80.7 & 81.6 & 83.9 & 88.4 & 70.7 & 71.2 & 89.1 & 93.5 \\
& PromptAD & CVPR'24 & 95.2 & 97.5 & 83.9 & 89.2 & \underline{87.6} & \textbf{92.6} & 91.0 & 94.4 & \textbf{81.0} & \textbf{87.3} & 88.2 & 92.4 \\
& AnomalyGPT & AAAI'24 & 95.9 & 98.0 & 87.2 & 92.6 & \textbf{88.5} & \underline{89.0} & 91.0 & 95.9 & 78.8 & \underline{85.4} & 86.2 & 87.8 \\
& ReMP-AD & ICCV'25 & 95.4 & \underline{98.2} & 89.5 & \underline{95.3} & 83.7 & 84.7 & 94.3 & 98.0 & 77.4 & 83.0 & 88.6 & \underline{94.3} \\
& KAG-prompt & AAAI'25 & 96.6 & \textbf{98.3} & 88.9 & 93.9 & 81.7 & 75.3 & 91.0 & 95.2 & 72.4 & 79.1 & 87.3 & 94.1 \\
& SubspaceAD$^\ddag$ & CVPR'26 & 96.9 & 97.6 & \underline{91.2} & 94.9 & 84.7 & 84.3 & 95.4 & \underline{98.7} & 74.1 & 72.3 & \underline{89.3} & 91.4 \\
& Hyper-FSAD$^\dag$ & \multicolumn{1}{c|}{-} & 94.7\textsubscript{$\pm$0.2} & 97.6\textsubscript{$\pm$0.2} & 90.8\textsubscript{$\pm$0.1} & \underline{95.3\textsubscript{$\pm$0.4}} & 86.4\textsubscript{$\pm$0.9} & 86.7\textsubscript{$\pm$1.7} & 91.3\textsubscript{$\pm$1.7} & 96.4\textsubscript{$\pm$0.7} & \underline{80.7\textsubscript{$\pm$0.5}} & 84.3\textsubscript{$\pm$1.4} & \textbf{89.9\textsubscript{$\pm$0.4}} & \textbf{94.9\textsubscript{$\pm$0.4}} \\
& Hyper-FSAD$^\ddag$ & \multicolumn{1}{c|}{-} & \textbf{97.4\textsubscript{$\pm$0.1}} & 98.1\textsubscript{$\pm$0.1} & \textbf{91.7\textsubscript{$\pm$0.2}} & \textbf{95.6\textsubscript{$\pm$0.2}} & 85.4\textsubscript{$\pm$0.8} & 84.6\textsubscript{$\pm$0.8} & \underline{95.7\textsubscript{$\pm$0.5}} & 98.5\textsubscript{$\pm$0.1} & 78.9\textsubscript{$\pm$0.4} & 70.2\textsubscript{$\pm$0.5} & 86.9\textsubscript{$\pm$0.3} & 84.7\textsubscript{$\pm$0.8} \\
& Hyper-FSAD$^\ast$ & \multicolumn{1}{c|}{-} & \underline{97.1\textsubscript{$\pm$0.0}} & \textbf{98.3\textsubscript{$\pm$0.0}} & \textbf{91.7\textsubscript{$\pm$0.1}} & 94.7\textsubscript{$\pm$0.3} & 82.6\textsubscript{$\pm$1.0} & 83.2\textsubscript{$\pm$0.8} & \textbf{96.2\textsubscript{$\pm$0.1}} & \textbf{98.9\textsubscript{$\pm$0.1}} & 79.6\textsubscript{$\pm$1.6} & 84.9\textsubscript{$\pm$0.9} & 88.6\textsubscript{$\pm$0.4} & 86.8\textsubscript{$\pm$0.7} \\
\bottomrule
\end{tabular}}
\label{tab:if1_iap_results}

\vspace{0pt}

\centering
\caption{Pixel-level PRO (\textbf{P-PRO}) and average precision (\textbf{P-AP}) across six datasets and 1/2/4-shot settings. Symbols, highlighting, backbone variants, and reporting protocol follow \cref{tab:auroc_results}.}
\renewcommand{\arraystretch}{0.77}%
    \resizebox{0.98\textwidth}{!}{%
\begin{tabular}{ll|l|cc|cc|cc|cc|cc|cc}
\toprule
& \multirow{2}{*}{Methods} & \multirow{2}{*}{Venue/Year} &
\multicolumn{2}{c|}{MVTecAD} & \multicolumn{2}{c|}{VisA} & \multicolumn{2}{c|}{MPDD} &
\multicolumn{2}{c|}{BTAD} & \multicolumn{2}{c|}{RESC} & \multicolumn{2}{c}{BraTS} \\
\cmidrule(lr){4-5}\cmidrule(lr){6-7}\cmidrule(lr){8-9}\cmidrule(lr){10-11}\cmidrule(lr){12-13}\cmidrule(lr){14-15}
& & & P-PRO & P-AP & P-PRO & P-AP & P-PRO & P-AP & P-PRO & P-AP & P-PRO & P-AP & P-PRO & P-AP \\
\midrule
\multirow{12}{*}{\rotatebox[origin=c]{90}{1-shot}}
& RegAD & ECCV'22 & 76.8 & 36.1 & 68.7 & 17.9 & 74.6 & 8.4 & 68.9 & 33.1 & 53.2 & 14.6 & 62.6 & 17.5 \\
& FastRecon & ICCV'23 & 90.8 & 50.5 & 85.2 & 26.0 & 89.3 & 26.0 & 80.6 & 60.0 & 83.6 & \underline{66.5} & 71.4 & 39.0 \\
& WinCLIP & CVPR'23 & 82.0 & 35.5 & 85.2 & 19.2 & 86.7 & 23.6 & 61.7 & 26.0 & 73.3 & 33.4 & 64.2 & 33.2 \\
& APRIL-GAN & CVPR'23 & 84.5 & 43.8 & 87.0 & 29.3 & 85.1 & 28.3 & 73.4 & 50.4 & 74.9 & 54.1 & 62.7 & 38.7 \\
& PromptAD & CVPR'24 & 90.9 & 53.3 & 88.4 & 29.1 & 90.4 & 30.1 & 79.4 & 61.3 & \underline{85.8} & \textbf{68.2} & 74.8 & 46.0 \\
& AnomalyGPT & AAAI'24 & 89.0 & 48.8 & 65.3 & 16.8 & 89.9 & 31.3 & 71.7 & 49.9 & 58.5 & 27.4 & 69.9 & 30.1 \\
& ReMP-AD & ICCV'25 & 90.7 & 56.1 & \underline{91.5} & 37.2 & \underline{92.2} & 34.5 & \textbf{82.8} & 58.1 & 71.2 & 43.9 & 64.3 & 41.3 \\
& KAG-prompt & AAAI'25 & 90.5 & 55.0 & 83.5 & 37.0 & 88.0 & 33.4 & 72.1 & 52.0 & 51.2 & 20.3 & \underline{77.8} & 55.3 \\
& SubspaceAD$^\ddag$ & CVPR'26 & \textbf{92.9} & 47.7 & \textbf{93.6} & 24.9 & \textbf{92.7} & 26.4 & 81.4 & 35.9 & 77.7 & 32.7 & \textbf{79.7} & 39.0 \\
& Hyper-FSAD$^\dag$ & \multicolumn{1}{c|}{-} & 87.0\textsubscript{$\pm$0.2} & 62.2\textsubscript{$\pm$0.1} & 87.0\textsubscript{$\pm$0.1} & \underline{38.0\textsubscript{$\pm$0.6}} & 89.9\textsubscript{$\pm$0.7} & \textbf{43.4\textsubscript{$\pm$1.1}} & 79.4\textsubscript{$\pm$0.4} & 63.0\textsubscript{$\pm$0.4} & 71.7\textsubscript{$\pm$1.4} & 59.9\textsubscript{$\pm$0.7} & 59.5\textsubscript{$\pm$0.1} & 54.9\textsubscript{$\pm$1.8} \\
& Hyper-FSAD$^\ddag$ & \multicolumn{1}{c|}{-} & \textbf{92.9\textsubscript{$\pm$0.1}} & \underline{65.5\textsubscript{$\pm$0.2}} & 91.0\textsubscript{$\pm$0.1} & 37.2\textsubscript{$\pm$0.3} & 91.7\textsubscript{$\pm$0.9} & \underline{39.8\textsubscript{$\pm$0.7}} & \underline{82.2\textsubscript{$\pm$0.2}} & \textbf{72.3\textsubscript{$\pm$0.1}} & 79.0\textsubscript{$\pm$0.7} & 45.5\textsubscript{$\pm$1.1} & 74.6\textsubscript{$\pm$0.6} & \underline{57.0\textsubscript{$\pm$0.9}} \\
& Hyper-FSAD$^\ast$ & \multicolumn{1}{c|}{-} & \underline{92.5\textsubscript{$\pm$0.1}} & \textbf{70.0\textsubscript{$\pm$0.1}} & 90.4\textsubscript{$\pm$0.1} & \textbf{40.7\textsubscript{$\pm$0.5}} & 89.2\textsubscript{$\pm$0.8} & 35.2\textsubscript{$\pm$0.4} & 80.9\textsubscript{$\pm$0.4} & \underline{70.6\textsubscript{$\pm$0.8}} & \textbf{86.1\textsubscript{$\pm$0.5}} & 62.7\textsubscript{$\pm$0.3} & 76.3\textsubscript{$\pm$0.6} & \textbf{62.7\textsubscript{$\pm$1.3}} \\
\midrule
\multirow{12}{*}{\rotatebox[origin=c]{90}{2-shot}}
& RegAD & ECCV'22 & 82.7 & 42.1 & 70.2 & 21.6 & 79.3 & 13.1 & 74.1 & 42.3 & 54.5 & 15.1 & 66.0 & 20.6 \\
& FastRecon & ICCV'23 & 91.5 & 51.9 & 85.2 & 30.6 & \underline{92.7} & 35.7 & 80.5 & 61.6 & 84.7 & \underline{68.4} & 71.7 & 35.3 \\
& WinCLIP & CVPR'23 & 82.7 & 37.4 & 85.9 & 23.6 & 89.4 & 26.8 & 63.4 & 27.5 & 74.6 & 35.7 & 63.6 & 32.9 \\
& APRIL-GAN & CVPR'23 & 85.5 & 45.1 & 86.8 & 30.1 & 86.6 & 30.2 & 73.2 & 50.8 & 76.5 & 56.0 & 63.1 & 38.8 \\
& PromptAD & CVPR'24 & 91.5 & 54.8 & 89.4 & 34.4 & 92.6 & 34.5 & 79.6 & 62.3 & \textbf{86.6} & \textbf{69.9} & 75.2 & 45.7 \\
& AnomalyGPT & AAAI'24 & 90.2 & 50.7 & 65.0 & 19.7 & 91.8 & 34.5 & 72.4 & 50.6 & 59.0 & 27.9 & 70.2 & 29.7 \\
& ReMP-AD & ICCV'25 & 91.7 & 58.6 & 92.0 & 38.8 & \textbf{93.2} & 37.4 & \underline{83.2} & 61.0 & 75.7 & 50.0 & 66.9 & 45.1 \\
& KAG-prompt & AAAI'25 & 91.1 & 56.7 & 85.2 & 37.5 & 88.1 & 34.2 & 71.8 & 52.5 & 50.4 & 21.8 & 78.1 & 58.0 \\
& SubspaceAD$^\ddag$ & CVPR'26 & \textbf{93.6} & 48.6 & \textbf{94.4} & 26.4 & 91.0 & 28.4 & 82.6 & 37.1 & 79.9 & 33.9 & \textbf{82.3} & 42.1 \\
& Hyper-FSAD$^\dag$ & \multicolumn{1}{c|}{-} & 87.7\textsubscript{$\pm$0.3} & 63.6\textsubscript{$\pm$0.4} & 88.3\textsubscript{$\pm$0.2} & \underline{40.8\textsubscript{$\pm$0.3}} & 89.3\textsubscript{$\pm$0.4} & \textbf{43.7\textsubscript{$\pm$0.8}} & 79.7\textsubscript{$\pm$0.3} & 64.9\textsubscript{$\pm$0.8} & 71.3\textsubscript{$\pm$0.5} & 61.3\textsubscript{$\pm$0.9} & 63.4\textsubscript{$\pm$1.9} & \underline{61.8\textsubscript{$\pm$2.0}} \\
& Hyper-FSAD$^\ddag$ & \multicolumn{1}{c|}{-} & \underline{93.1\textsubscript{$\pm$0.2}} & \underline{66.0\textsubscript{$\pm$0.2}} & \underline{92.2\textsubscript{$\pm$0.1}} & 39.8\textsubscript{$\pm$0.6} & 90.5\textsubscript{$\pm$0.7} & \underline{40.1\textsubscript{$\pm$0.9}} & \textbf{83.3\textsubscript{$\pm$0.6}} & \textbf{73.8\textsubscript{$\pm$0.6}} & 81.3\textsubscript{$\pm$0.3} & 48.2\textsubscript{$\pm$1.2} & 77.3\textsubscript{$\pm$1.4} & 59.6\textsubscript{$\pm$1.8} \\
& Hyper-FSAD$^\ast$ & \multicolumn{1}{c|}{-} & 92.9\textsubscript{$\pm$0.2} & \textbf{70.8\textsubscript{$\pm$0.1}} & 91.5\textsubscript{$\pm$0.1} & \textbf{42.7\textsubscript{$\pm$0.2}} & 88.6\textsubscript{$\pm$0.3} & 37.9\textsubscript{$\pm$0.9} & 82.1\textsubscript{$\pm$0.2} & \underline{72.1\textsubscript{$\pm$0.7}} & \underline{86.3\textsubscript{$\pm$1.4}} & 64.8\textsubscript{$\pm$1.1} & \underline{79.2\textsubscript{$\pm$1.5}} & \textbf{67.9\textsubscript{$\pm$1.0}} \\
\midrule
\multirow{12}{*}{\rotatebox[origin=c]{90}{4-shot}}
& RegAD & ECCV'22 & 86.0 & 46.5 & 72.8 & 21.4 & 83.3 & 16.4 & 75.5 & 44.1 & 60.0 & 18.1 & 70.2 & 24.8 \\
& FastRecon & ICCV'23 & 92.2 & 53.9 & 86.2 & 32.5 & 93.1 & 37.8 & 80.8 & 62.2 & 82.8 & \underline{68.5} & 73.8 & 43.9 \\
& WinCLIP & CVPR'23 & 83.8 & 39.2 & 86.5 & 25.7 & 90.7 & 29.3 & 64.7 & 28.5 & 75.7 & 38.4 & 63.0 & 40.0 \\
& APRIL-GAN & CVPR'23 & 86.6 & 46.6 & 86.6 & 30.6 & 86.9 & 31.4 & 74.6 & 50.9 & 77.6 & 57.3 & 63.0 & 40.0 \\
& PromptAD & CVPR'24 & 92.4 & 57.5 & 89.5 & 37.5 & \underline{94.0} & 40.5 & 80.1 & 62.5 & \textbf{86.8} & \textbf{71.3} & 77.0 & 54.4 \\
& AnomalyGPT & AAAI'24 & 91.2 & 52.9 & 65.4 & 20.8 & 93.2 & 40.8 & 73.5 & 50.6 & 60.0 & 28.5 & 73.6 & 41.8 \\
& ReMP-AD & ICCV'25 & 92.6 & 61.4 & 92.3 & 41.6 & \textbf{94.2} & 39.7 & \underline{83.2} & 61.6 & 77.5 & 53.7 & 69.5 & 48.2 \\
& KAG-prompt & AAAI'25 & 91.4 & 57.4 & 85.5 & 38.8 & 90.2 & 35.4 & 70.9 & 52.5 & 53.9 & 24.1 & \underline{78.6} & 59.5 \\
& SubspaceAD$^\ddag$ & CVPR'26 & \textbf{93.9} & 50.4 & \textbf{94.9} & 27.0 & 91.7 & 28.6 & \textbf{83.4} & 36.6 & 81.8 & 36.4 & \textbf{79.9} & 45.4 \\
& Hyper-FSAD$^\dag$ & \multicolumn{1}{c|}{-} & 88.4\textsubscript{$\pm$0.2} & 64.9\textsubscript{$\pm$0.3} & 89.3\textsubscript{$\pm$0.2} & \underline{41.8\textsubscript{$\pm$0.5}} & 90.3\textsubscript{$\pm$0.4} & \textbf{44.8\textsubscript{$\pm$1.0}} & 80.8\textsubscript{$\pm$0.3} & 66.5\textsubscript{$\pm$0.6} & 72.1\textsubscript{$\pm$0.6} & 61.7\textsubscript{$\pm$0.8} & 64.0\textsubscript{$\pm$1.1} & \textbf{62.0\textsubscript{$\pm$2.3}} \\
& Hyper-FSAD$^\ddag$ & \multicolumn{1}{c|}{-} & \underline{93.5\textsubscript{$\pm$0.1}} & \underline{66.2\textsubscript{$\pm$0.1}} & \underline{93.2\textsubscript{$\pm$0.1}} & 41.2\textsubscript{$\pm$0.2} & 91.4\textsubscript{$\pm$0.6} & \underline{41.4\textsubscript{$\pm$0.9}} & \textbf{83.4\textsubscript{$\pm$0.6}} & \textbf{74.2\textsubscript{$\pm$0.4}} & 80.9\textsubscript{$\pm$1.1} & 48.3\textsubscript{$\pm$1.1} & 77.4\textsubscript{$\pm$0.8} & 52.3\textsubscript{$\pm$1.6} \\
& Hyper-FSAD$^\ast$ & \multicolumn{1}{c|}{-} & \underline{93.5\textsubscript{$\pm$0.1}} & \textbf{71.5\textsubscript{$\pm$0.1}} & 92.5\textsubscript{$\pm$0.1} & \textbf{43.2\textsubscript{$\pm$0.2}} & 89.3\textsubscript{$\pm$0.3} & 40.3\textsubscript{$\pm$1.6} & 82.2\textsubscript{$\pm$0.4} & \underline{73.0\textsubscript{$\pm$0.4}} & \underline{83.2\textsubscript{$\pm$1.0}} & 64.1\textsubscript{$\pm$2.1} & 77.4\textsubscript{$\pm$1.4} & \underline{61.5\textsubscript{$\pm$1.7}} \\
\bottomrule
\end{tabular}}
\label{tab:ppro_pap_results}
\end{table*}

\subsection{Cross-Dataset Patch-Weighting Ablation}
\label{sec:cross_dataset_lookup_ablation}

We further report the patch-weighting comparison on MPDD, BTAD, RESC, and BraTS, complementing the MVTecAD and VisA results in \cref{tab:lookup_ablation}. All configurations use the same default DINOv3 ViT-B/16 backbone and $448\times448$ input resolution. For consistency, every sparsemax entry is identical to the corresponding DINOv3 result in \cref{tab:auroc_results}, and standard deviations are omitted.

\subsection{Cross-Dataset DINOv3 Layer Selection}
\label{sec:cross_dataset_layer_selection}

We further report the layer-selection comparison on MPDD, BTAD, RESC, and BraTS, complementing the MVTecAD and VisA results in \cref{tab:layer_selection_main}. All configurations use DINOv3 ViT-B/16 with $448\times448$ inputs. The $\{3,6,9,12\}$ entries are aligned with the corresponding DINOv3 results in \cref{tab:auroc_results}, and standard deviations are omitted.

\begin{table*}[!t]
\centering
\caption{Cross-dataset patch-weighting ablation across 1/2/4-shot settings. I/P denote I-AUROC/P-AUROC (\%); sparsemax values match \cref{tab:auroc_results}, and standard deviations are omitted.}
\label{tab:cross_dataset_lookup_1shot}
\label{tab:cross_dataset_lookup_2shot}
\label{tab:cross_dataset_lookup_4shot}
\setlength{\tabcolsep}{3.2pt}
\renewcommand{\arraystretch}{1.02}
\small
\resizebox{\textwidth}{!}{%
\begin{tabular}{ll|cccc|cccc|cccc}
\toprule
\multirow{2}{*}{Dataset} & \multirow{2}{*}{Metric} &
\multicolumn{4}{c|}{1-shot} & \multicolumn{4}{c|}{2-shot} & \multicolumn{4}{c}{4-shot} \\
\cmidrule(lr){3-6}\cmidrule(lr){7-10}\cmidrule(lr){11-14}
& & Softmax & Entmax & Top-$p$ & Sparsemax & Softmax & Entmax & Top-$p$ & Sparsemax & Softmax & Entmax & Top-$p$ & Sparsemax \\
\midrule
\multirow{2}{*}{MPDD} & I & 67.1 & 78.6 & 74.6 & \textbf{78.8} & 68.0 & 80.1 & 76.8 & \textbf{80.3} & 67.7 & 80.1 & 77.0 & \textbf{80.3} \\
& P & 89.0 & 95.3 & 94.6 & \textbf{96.4} & 88.6 & 95.7 & 95.0 & \textbf{96.4} & 88.4 & 95.9 & 95.3 & \textbf{96.5} \\
\multirow{2}{*}{BTAD} & I & 83.7 & 95.2 & 91.2 & \textbf{95.4} & 83.5 & 95.6 & 92.3 & \textbf{95.8} & 83.6 & 96.0 & 92.9 & \textbf{96.2} \\
& P & 90.5 & 96.8 & 96.1 & \textbf{97.9} & 90.2 & 97.3 & 96.6 & \textbf{98.0} & 90.1 & 97.6 & 97.0 & \textbf{98.2} \\
\multirow{2}{*}{RESC} & I & 81.3 & 92.8 & 88.8 & \textbf{93.0} & 79.2 & 91.3 & 88.0 & \textbf{91.5} & 76.5 & 88.9 & 85.8 & \textbf{89.1} \\
& P & 88.9 & 95.2 & 94.5 & \textbf{96.3} & 88.7 & 95.8 & 95.1 & \textbf{96.5} & 87.8 & 95.3 & 94.7 & \textbf{95.9} \\
\multirow{2}{*}{BraTS} & I & 70.5 & 82.0 & 78.0 & \textbf{82.2} & 70.9 & 83.0 & 79.7 & \textbf{83.2} & 67.9 & 80.3 & 77.2 & \textbf{80.5} \\
& P & 88.6 & 94.9 & 94.2 & \textbf{96.0} & 89.1 & 96.2 & 95.5 & \textbf{96.9} & 88.4 & 95.9 & 95.3 & \textbf{96.5} \\
\bottomrule
\end{tabular}%
}

\vspace{8pt}

\centering
\caption{Cross-dataset DINOv3 layer-selection analysis across 1/2/4-shot settings. I/P denote I-AUROC/P-AUROC (\%); standard deviations are omitted.}
\label{tab:cross_dataset_layers_1shot}
\label{tab:cross_dataset_layers_2shot}
\label{tab:cross_dataset_layers_4shot}
\setlength{\tabcolsep}{3.2pt}
\renewcommand{\arraystretch}{1.02}
\small
\resizebox{\textwidth}{!}{%
\begin{tabular}{ll|cccc|cccc|cccc}
\toprule
\multirow{2}{*}{Dataset} & \multirow{2}{*}{Metric} &
\multicolumn{4}{c|}{1-shot} & \multicolumn{4}{c|}{2-shot} & \multicolumn{4}{c}{4-shot} \\
\cmidrule(lr){3-6}\cmidrule(lr){7-10}\cmidrule(lr){11-14}
& & $\{6\}$ & $\{12\}$ & $\{6,12\}$ & $\{3,6,9,12\}$ & $\{6\}$ & $\{12\}$ & $\{6,12\}$ & $\{3,6,9,12\}$ & $\{6\}$ & $\{12\}$ & $\{6,12\}$ & $\{3,6,9,12\}$ \\
\midrule
\multirow{2}{*}{MPDD} & I & 75.4 & 72.3 & 76.6 & \textbf{78.8} & 77.3 & 75.1 & 78.8 & \textbf{80.3} & 77.5 & 76.1 & 78.9 & \textbf{80.3} \\
& P & 94.9 & 93.2 & 96.1 & \textbf{96.4} & 95.2 & 94.0 & 96.2 & \textbf{96.4} & 95.4 & 94.1 & 96.3 & \textbf{96.5} \\
\multirow{2}{*}{BTAD} & I & 92.0 & 88.9 & 93.2 & \textbf{95.4} & 92.8 & 90.6 & 94.3 & \textbf{95.8} & 93.4 & 92.0 & 94.8 & \textbf{96.2} \\
& P & 96.4 & 94.7 & 97.6 & \textbf{97.9} & 96.8 & 95.6 & 97.8 & \textbf{98.0} & 97.1 & 95.8 & 98.0 & \textbf{98.2} \\
\multirow{2}{*}{RESC} & I & 89.6 & 86.5 & 90.8 & \textbf{93.0} & 88.5 & 86.3 & 90.0 & \textbf{91.5} & 86.3 & 84.9 & 87.7 & \textbf{89.1} \\
& P & 94.8 & 93.1 & 96.0 & \textbf{96.3} & 95.3 & 94.1 & 96.3 & \textbf{96.5} & 94.8 & 93.5 & 95.7 & \textbf{95.9} \\
\multirow{2}{*}{BraTS} & I & 78.8 & 75.7 & 80.0 & \textbf{82.2} & 80.2 & 78.0 & 81.7 & \textbf{83.2} & 77.7 & 76.3 & 79.1 & \textbf{80.5} \\
& P & 94.5 & 92.8 & 95.7 & \textbf{96.0} & 95.7 & 94.5 & 96.7 & \textbf{96.9} & 95.4 & 94.1 & 96.3 & \textbf{96.5} \\
\bottomrule
\end{tabular}%
}
\end{table*}
\begin{figure*}[!t]
    \centering
    \begin{subfigure}[t]{0.245\textwidth}
        \centering
        \includegraphics[width=\linewidth]{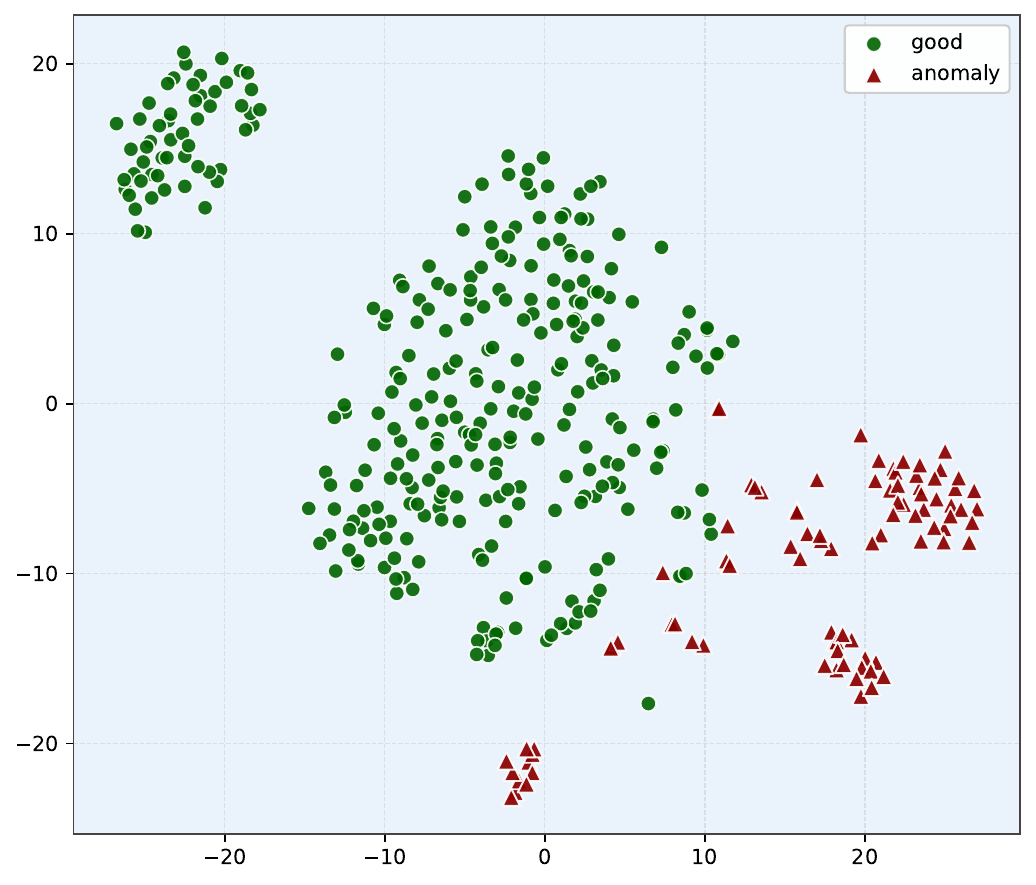}
        \caption{\texttt{carpet}}
    \end{subfigure}\hfill
    \begin{subfigure}[t]{0.245\textwidth}
        \centering
        \includegraphics[width=\linewidth]{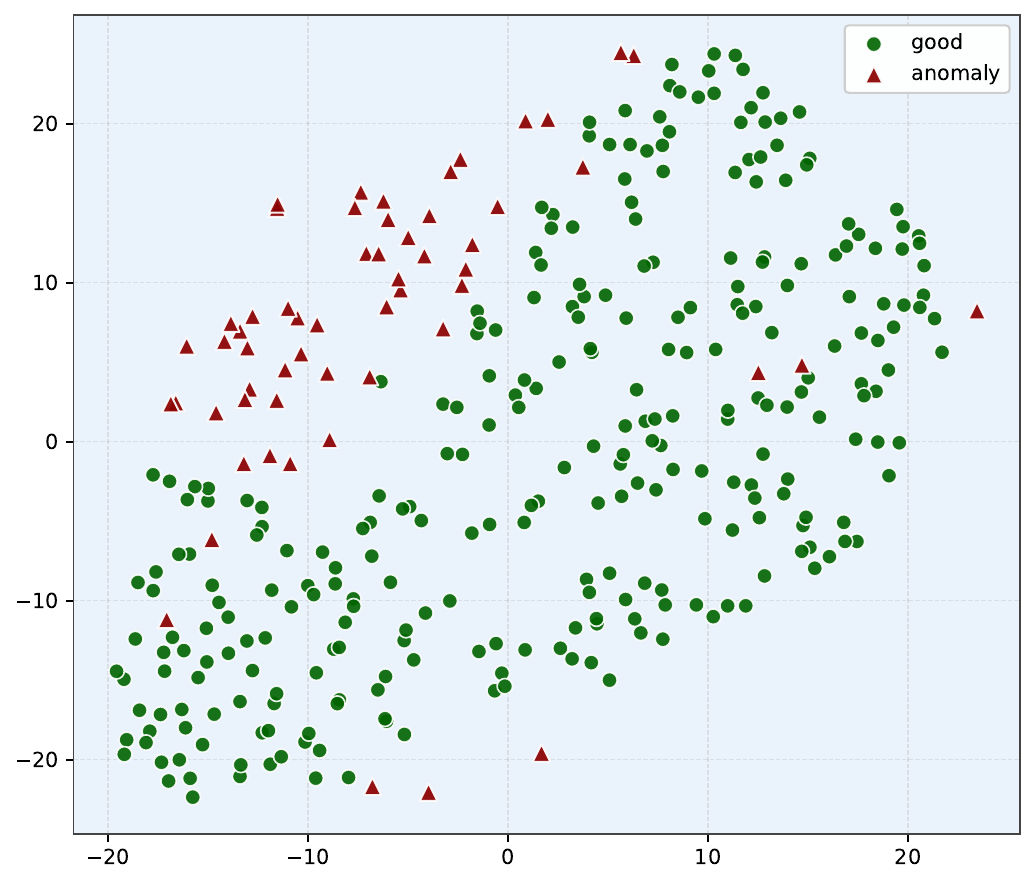}
        \caption{\texttt{grid}}
    \end{subfigure}\hfill
    \begin{subfigure}[t]{0.245\textwidth}
        \centering
        \includegraphics[width=\linewidth]{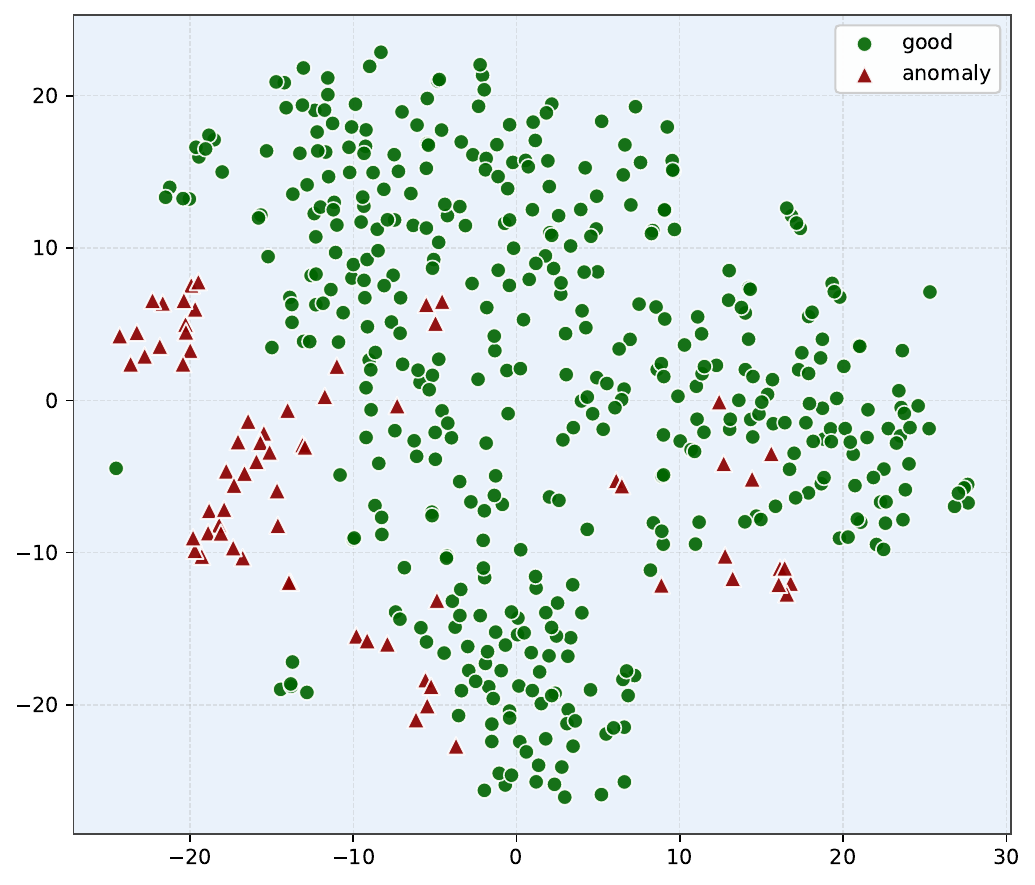}
        \caption{\texttt{hazelnut}}
    \end{subfigure}\hfill
    \begin{subfigure}[t]{0.245\textwidth}
        \centering
        \includegraphics[width=\linewidth]{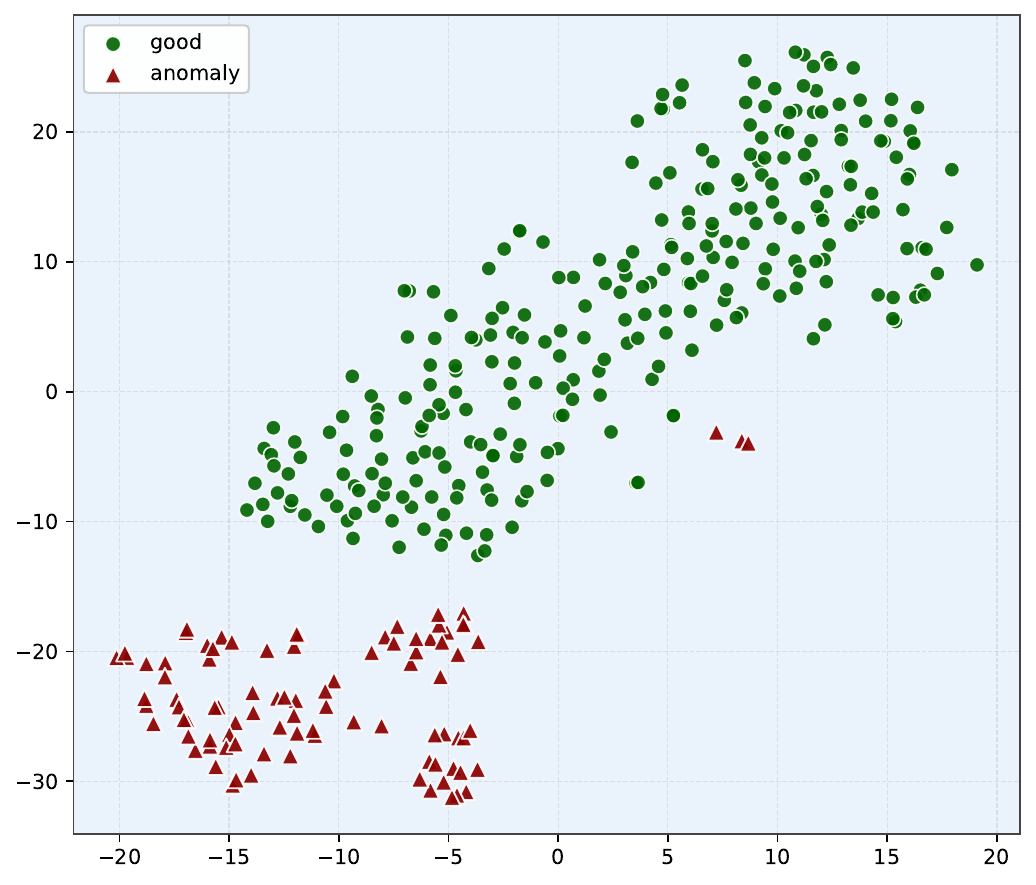}
        \caption{\texttt{leather}}
    \end{subfigure}

    \vspace{2mm}

    \begin{subfigure}[t]{0.245\textwidth}
        \centering
        \includegraphics[width=\linewidth]{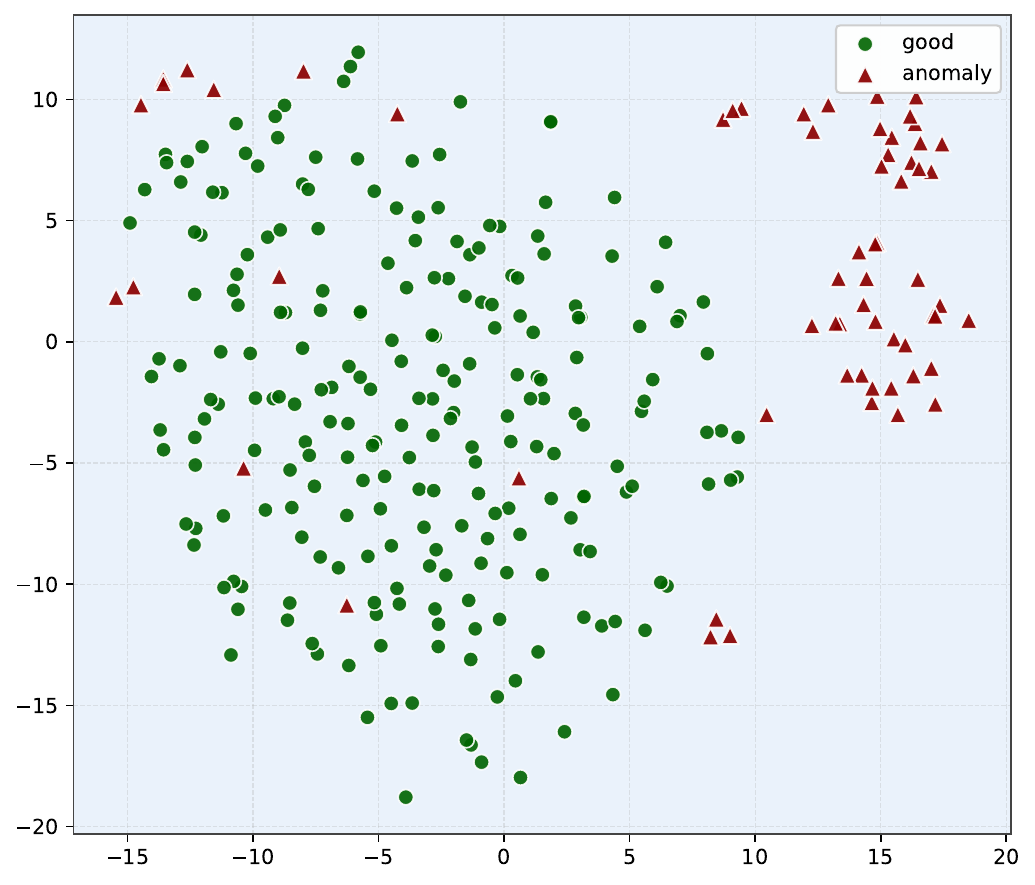}
        \caption{\texttt{bottle}}
    \end{subfigure}\hfill
    \begin{subfigure}[t]{0.245\textwidth}
        \centering
        \includegraphics[width=\linewidth]{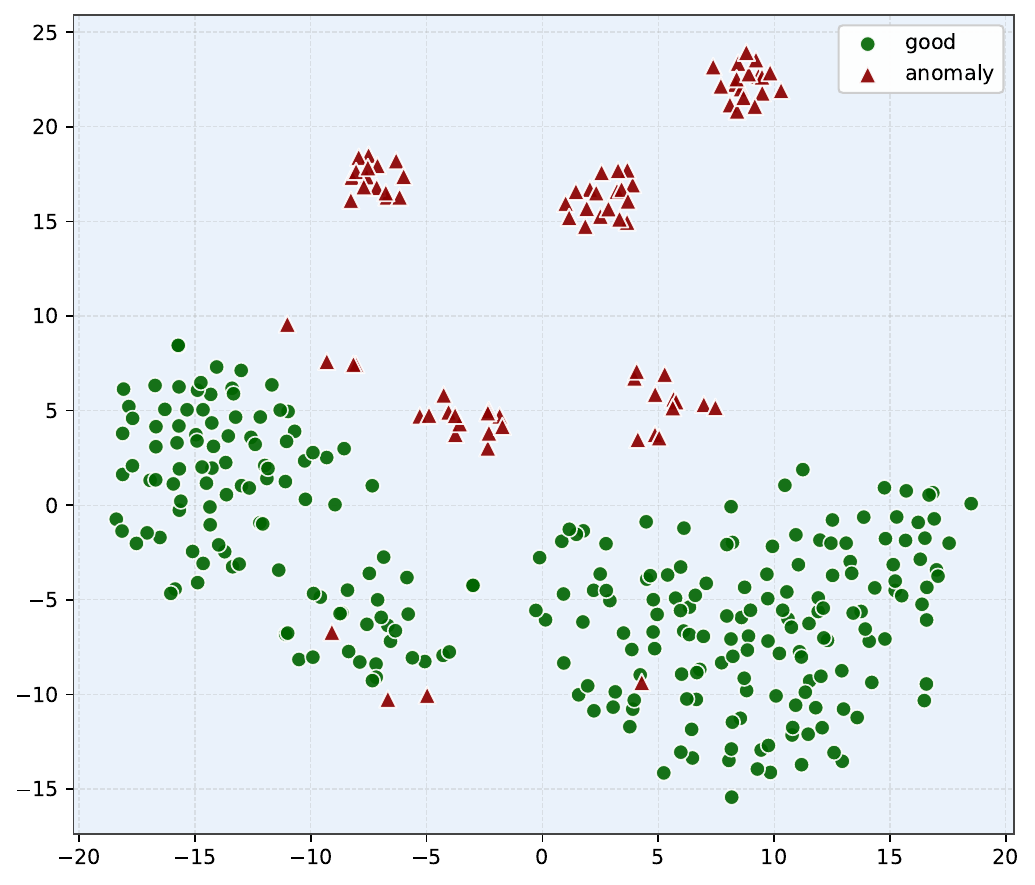}
        \caption{\texttt{tile}}
    \end{subfigure}\hfill
    \begin{subfigure}[t]{0.245\textwidth}
        \centering
        \includegraphics[width=\linewidth]{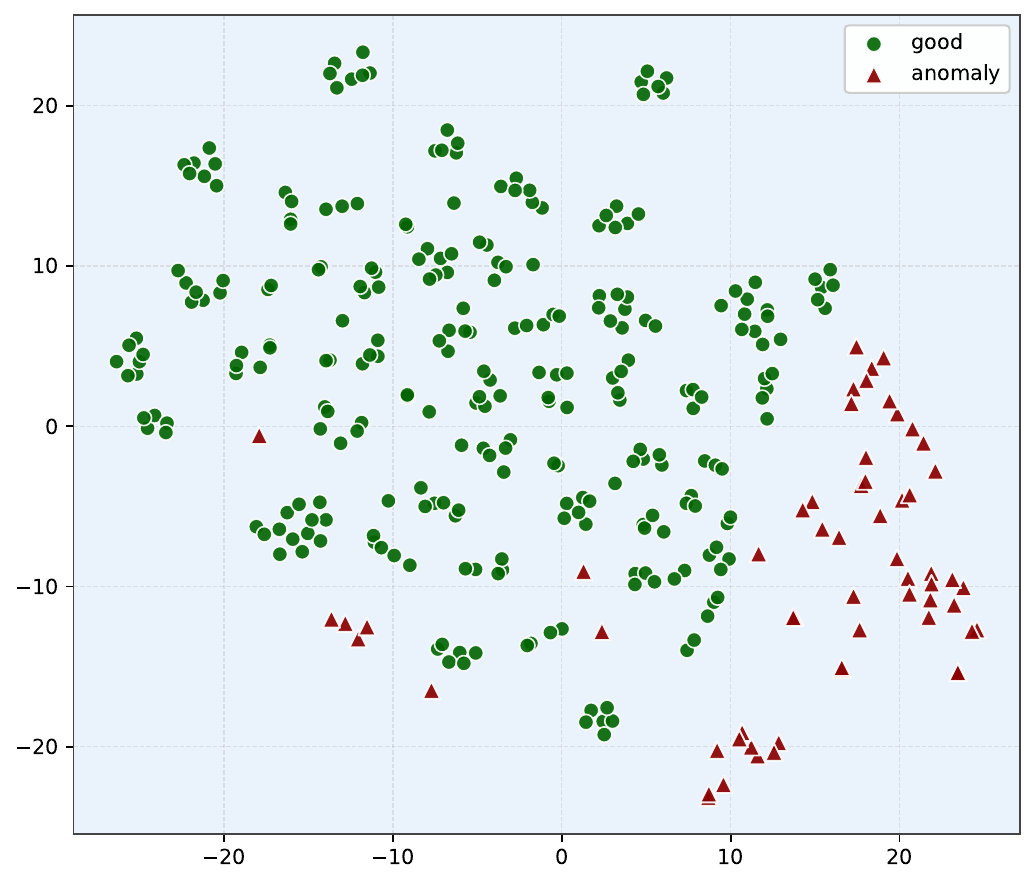}
        \caption{\texttt{wood}}
    \end{subfigure}\hfill
    \begin{subfigure}[t]{0.245\textwidth}
        \centering
        \includegraphics[width=\linewidth]{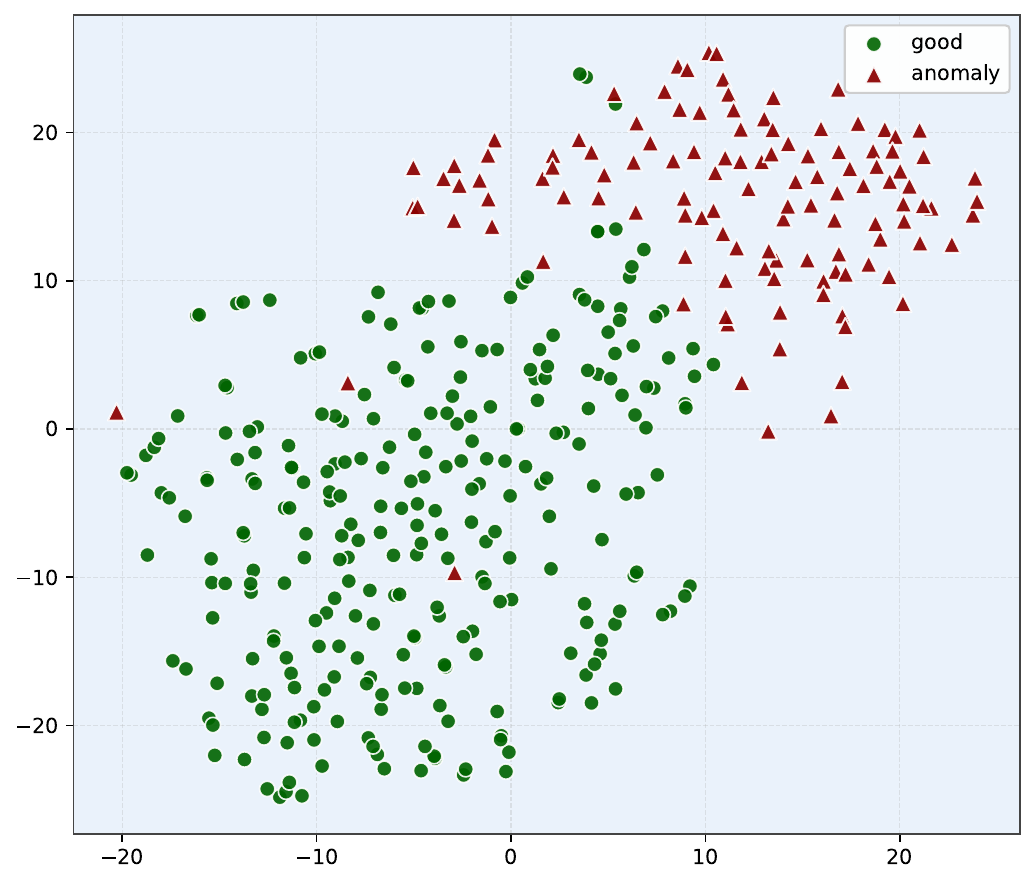}
        \caption{\texttt{zipper}}
    \end{subfigure}

    \caption{\textbf{t-SNE visualization of selected-layer dinov3\_vitb16 <CLS> embeddings} for normal vs. anomalous samples across eight MVTecAD categories.}
    \label{fig:tsne_cls_mvtec8}
\end{figure*}

\section{Effect of the Fusion Strategy}
\label{sec:fusion_appendix}
\cref{tab:lambda_ablation_shots} studies the fusion weight $\lambda$ in our Dual-Branch Image Scoring, where $\lambda{=}0$ uses only the Global Semantic Branch ($S_{\mathrm{cls}}$), $\lambda{=}1$ uses only the Spatial Evidence Branch ($S_{\mathrm{map}}$), and $\lambda{=}0.5$ fuses both. When relying solely on global <CLS> semantics ($\lambda{=}0$), the method remains functional but clearly suboptimal, suggesting that support-to-query <CLS> cosine matching encodes a meaningful normality prior, yet lacks the spatial evidence needed for robust anomaly assessment. This observation is consistent with \cref{fig:tsne_cls_mvtec8}, where normal and anomalous samples are well separated in the <CLS> embedding space across multiple categories, indicating discriminative global representations while not guaranteeing accurate localization-aware scoring. Conversely, using only spatial cues from the anomaly map ($\lambda{=}1$) performs strongly (1-shot: 95.7\%/97.1\%/98.1\%) but is still inferior to the fusion, as isolated high responses or local clutter may dominate map pooling. Notably, the balanced fusion $\lambda{=}0.5$ consistently achieves the best results across all shots, reaching 96.3\%/97.6\%/98.4\% in 1-shot and 96.9\%/97.9\%/98.7\% in 4-shot, confirming that global <CLS>-level semantics and spatial localization cues are complementary: the former stabilizes image-level decisions under distributional variations, while the latter injects fine-grained spatial evidence to improve reliability in challenging cases.

\section{Effect of the Pooling Operator}
\label{sec:pooling_appendix}
\cref{tab:pooling_ablation_shots} evaluates different pooling operators for reducing the anomaly map $A_{map}$ into the map-based image score $S_{\mathrm{map}}$. Overall, \textbf{max} pooling and Top-$1\%$ pooling yield very similar performance across shots. In the 1-shot setting, Top-$1\%$ is slightly better in I-AUROC (96.5\% vs.\ 96.3\%) and I-F1 (97.7\% vs.\ 97.6\%), while tying I-AP at 98.4\%. This suggests that the anomaly map produced by our Sparse Hyper Matching is already highly concentrated on truly anomalous regions and thus benefits only marginally from aggregating multiple high-response locations. In contrast, Top-$10$ pooling is consistently inferior across all shots, indicating that fixed-cardinality pooling is less reliable in our setting, likely because the number of informative high-score patches varies with anomaly scale and background complexity, so a fixed $n$ can either include spurious responses (hurting specificity) or under-summarize extended anomalies (hurting sensitivity). Despite the small gain of Top-$1\%$ over max pooling on MVTecAD, we adopt \textbf{max} pooling as the default to improve robustness in practice: percentile-based pooling introduces an extra ratio hyperparameter whose optimal value can vary with image resolution, category, and anomaly size, whereas max pooling is hyperparameter-free and thus avoids dataset-specific tuning while remaining highly competitive (\eg, 4-shot: 96.9\%/97.9\%/98.7\% for I-AUROC/I-F1-max/I-AP).

\section{More Qualitative Results}
\label{sec:more_qualitative}

\cref{fig:more_qualitative_results_1,fig:more_qualitative_results_2,fig:more_qualitative_results_3,fig:more_qualitative_results_4,fig:more_qualitative_results_5} provide additional 1-shot comparisons between ground-truth (GT) annotations and anomaly maps predicted by Hyper-FSAD$^\ast$. Across industrial and medical images, Hyper-FSAD$^\ast$ consistently concentrates high anomaly responses within the annotated regions, covering compact, elongated, multiple, and diffuse anomalies. Together, they demonstrate robust localization across diverse anomalies without target-task training or language supervision.

\begin{table*}[!t]
\begin{minipage}[t]{0.485\textwidth}
\centering
\caption{The effects of $\lambda$ across different shots on MVTecAD~\cite{bergmann2019mvtec}, evaluated by \textbf{I-AUROC}~(\%), \textbf{I-F1-max}~(\%), and \textbf{I-AP}~(\%).}
\label{tab:lambda_ablation_shots}
\setlength{\tabcolsep}{6pt}
\renewcommand{\arraystretch}{1.08}
\small
\resizebox{0.9\linewidth}{!}{%
\begin{tabular}{c ccc ccc ccc}
\toprule
\multirow{2}{*}{$\lambda$} & \multicolumn{3}{c}{1-shot} & \multicolumn{3}{c}{2-shot} & \multicolumn{3}{c}{4-shot} \\
\cmidrule(lr){2-4}\cmidrule(lr){5-7}\cmidrule(lr){8-10}
& I-AUROC & I-F1 & I-AP & I-AUROC & I-F1 & I-AP & I-AUROC & I-F1 & I-AP \\
\midrule
0   & 90.3 & 93.2 & 94.3 & 91.2 & 94.0 & 94.7 & 91.5 & 94.3 & 95.0 \\
0.5 & \textbf{96.3} & \textbf{97.6} & \textbf{98.4} & \textbf{96.8} & \textbf{97.7} & \textbf{98.5} & \textbf{96.9} & \textbf{97.9} & \textbf{98.7} \\
1   & 95.7 & 97.1 & 98.1 & 96.0 & 97.3 & 98.3 & 96.3 & 97.6 & 98.5 \\
\bottomrule
\end{tabular}%
}
\end{minipage}\hfill
\begin{minipage}[t]{0.485\textwidth}
\centering
\caption{The effects of different pooling operators across shots on MVTecAD~\cite{bergmann2019mvtec}, evaluated by \textbf{I-AUROC}~(\%), \textbf{I-F1}~(\%), and \textbf{I-AP}~(\%).}
\label{tab:pooling_ablation_shots}
\setlength{\tabcolsep}{6pt}
\renewcommand{\arraystretch}{1.08}
\small
\resizebox{0.9\linewidth}{!}{%
\begin{tabular}{c ccc ccc ccc}
\toprule
\multirow{2}{*}{Pooling operator} & \multicolumn{3}{c}{1-shot} & \multicolumn{3}{c}{2-shot} & \multicolumn{3}{c}{4-shot} \\
\cmidrule(lr){2-4}\cmidrule(lr){5-7}\cmidrule(lr){8-10}
& I-AUROC & I-F1 & I-AP & I-AUROC & I-F1 & I-AP & I-AUROC & I-F1 & I-AP \\
\midrule
Top 10  & 95.3 & 95.9 & 97.5 & 95.8 & 96.6 & 97.8 & 95.9 & 97.0 & 97.9 \\
Top 1\% & \textbf{96.5} & \textbf{97.7} & \textbf{98.4} & \textbf{96.9} & 97.6 & \textbf{98.5} & 96.8 & \textbf{97.9} & 98.5 \\
max     & 96.3 & 97.6 & 98.4 & 96.8 & \textbf{97.7} & \textbf{98.5} & \textbf{96.9} & \textbf{97.9} & \textbf{98.7} \\
\bottomrule
\end{tabular}%
}
\end{minipage}
\end{table*}

\begin{figure*}[!t]
    \centering
    \includegraphics[width=0.98\textwidth]{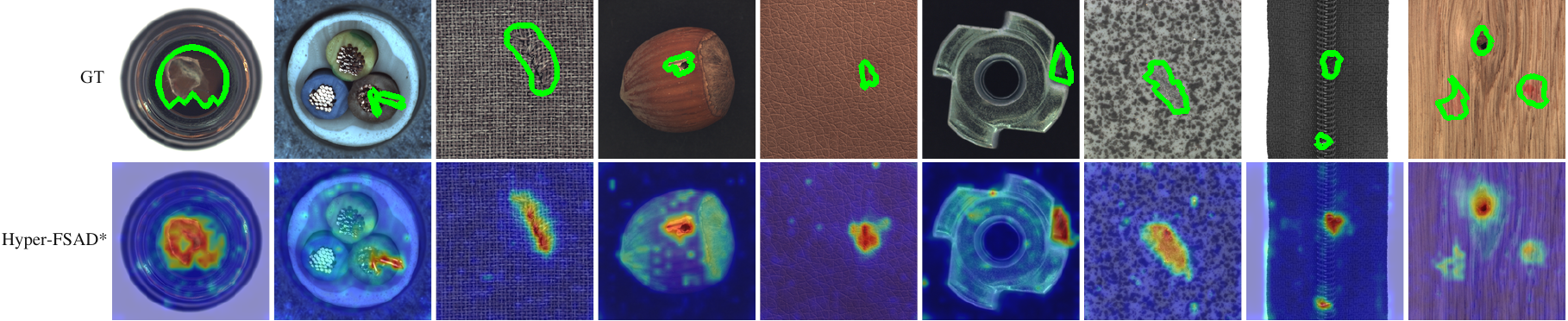}
    \caption{Additional 1-shot comparison between GT annotations (top) and Hyper-FSAD$^\ast$ anomaly maps (bottom), set 1. Hyper-FSAD$^\ast$ denotes our DINOv3 ViT-B/16 instantiation.}
    \label{fig:more_qualitative_results_1}
\end{figure*}

\begin{figure*}[!t]
    \centering
    \includegraphics[width=0.98\textwidth]{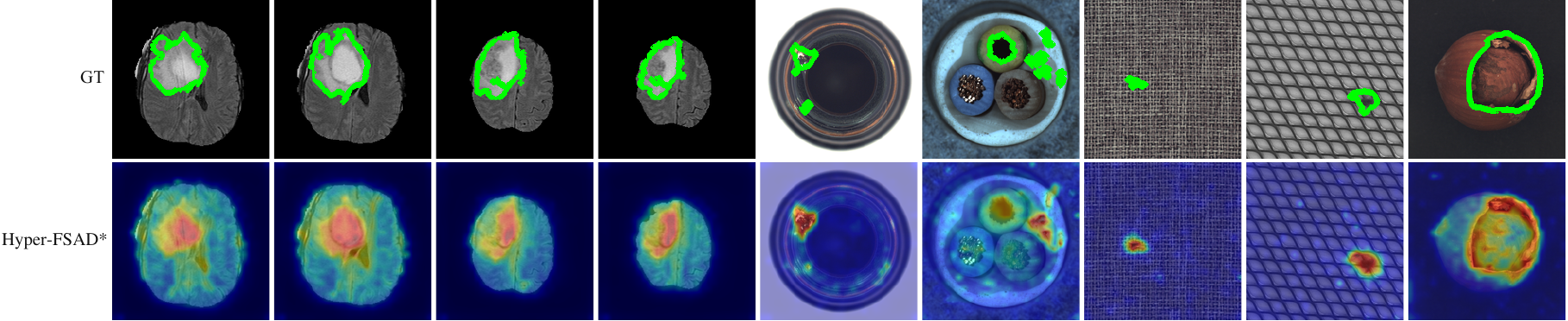}
    \caption{Additional 1-shot comparison between GT annotations (top) and Hyper-FSAD$^\ast$ anomaly maps (bottom), set 2.}
    \label{fig:more_qualitative_results_2}
\end{figure*}
\FloatBarrier

\begin{figure*}[p]
    \centering
    \includegraphics[width=0.98\textwidth]{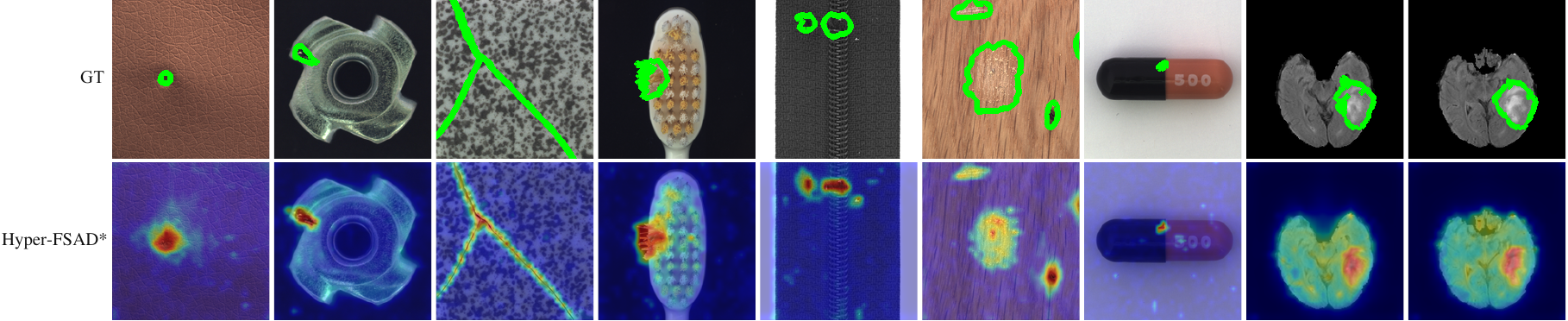}
    \caption{Additional 1-shot comparison between GT annotations (top) and Hyper-FSAD$^\ast$ anomaly maps (bottom), set 3.}
    \label{fig:more_qualitative_results_3}
\end{figure*}

\begin{figure*}[p]
    \centering
    \includegraphics[width=0.98\textwidth]{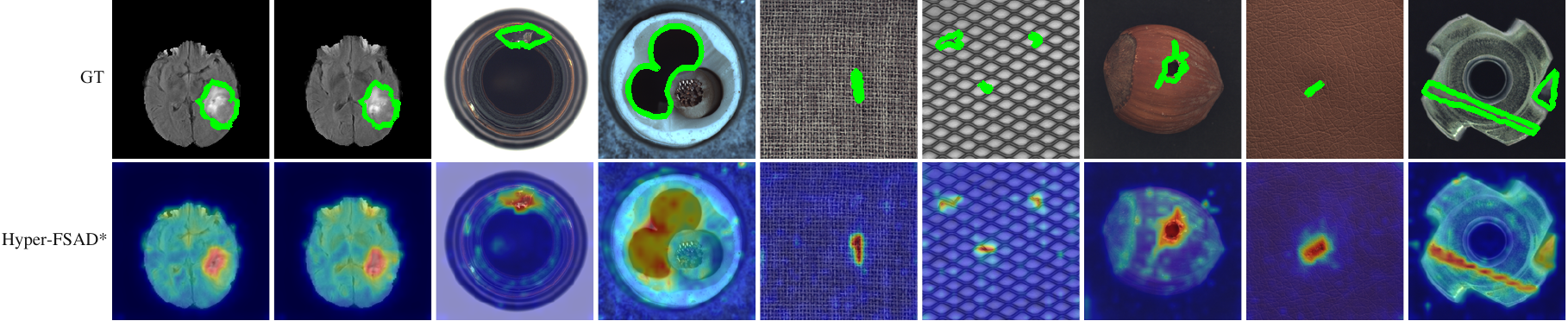}
    \caption{Additional 1-shot comparison between GT annotations (top) and Hyper-FSAD$^\ast$ anomaly maps (bottom), set 4.}
    \label{fig:more_qualitative_results_4}
\end{figure*}

\makeatletter
\setlength{\@dblfptop}{0pt}
\makeatother
\begin{figure*}[p]
    \centering
    \includegraphics[width=0.98\textwidth]{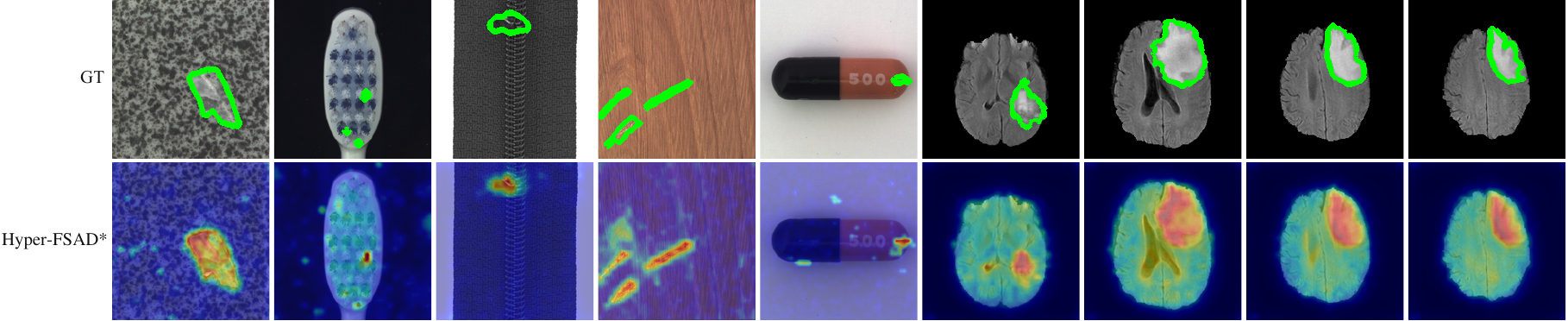}
    \caption{Additional 1-shot comparison between GT annotations (top) and Hyper-FSAD$^\ast$ anomaly maps (bottom), set 5.}
    \label{fig:more_qualitative_results_5}
\end{figure*}

\end{document}